\documentclass[acmsmall]{acmart}

\setcopyright{none}
\copyrightyear{}
\acmYear{}
\acmDOI{}
\acmPrice{}
\acmISBN{}
\acmConference{}
\acmBooktitle{}
\received{}

\usepackage{amsmath,amssymb,amsfonts}
\usepackage{algorithmic}
\usepackage{listings}
\usepackage{graphicx}
\usepackage{textcomp}
\usepackage{subfigure}
\usepackage{xcolor}
\usepackage{tcolorbox}

\usepackage{tcolorbox} 

\newcommand{\R}{\mathbb{R}}

\newcommand{\cY}{\mathcal{Y}}
\newcommand{\cX}{\mathcal{X}}
\newcommand{\cB}{\mathcal{B}}

\def\BibTeX{{\rm B\kern-.05em{\sc i\kern-.025em b}\kern-.08em
    T\kern-.1667em\lower.7ex\hbox{E}\kern-.125emX}}

\usepackage{multirow}

\newtheorem{prop}{Proposition}

\newtheorem{rmk}{Remark}
\newtheorem{defn}{Definition}

\lstdefinestyle{mystyle}{
    backgroundcolor=\color{white},   
    commentstyle=\color{gray},
    keywordstyle=\color{blue},
    numberstyle=\tiny\color{gray},
    stringstyle=\color{red},
    basicstyle=\ttfamily\footnotesize,
    breakatwhitespace=false,         
    breaklines=true,                 
    captionpos=b,                    
    keepspaces=true,                 
    numbers=left,                    
    numbersep=5pt,                  
    showspaces=false,                
    showstringspaces=false,
    showtabs=false,                  
    tabsize=2
}

\usepackage[commandnameprefix=always]{changes}

\renewcommand\footnotetextcopyrightpermission[1]{}

\begin{document}


\title{Optimizing Tensor Train Decomposition in DNNs for RISC-V Architectures Using Design Space Exploration and Compiler Optimizations}

\author{Theologos Anthimopoulos}
\authornote{Both authors contributed equally to this research.}

\affiliation{%
  \institution{School of Informatics, Aristotle University of Thessaloniki}
  \city{Thessaloniki}
  \country{Greece}
}
\email{tanthimop@csd.auth.gr}
\orcid{0009-0007-3620-4637}

\author{Milad Kokhazadeh}
\authornotemark[1]

\affiliation{%
  \institution{School of Informatics, Aristotle University of Thessaloniki}
  \city{Thessaloniki}
  \country{Greece}
}
\email{Kokhazad@csd.auth.gr}
\orcid{0009-0001-7430-3691}

\author{Vasilios Kelefouras}
\affiliation{%
  \institution{School of Engineering, Computing and Mathematics, University of Plymouth}
  \city{Plymouth}
  \country{United Kingdom}
}
\email{vasilios.kelefouras@plymouth.ac.uk}
\orcid{0000-0001-9591-913X}

\author{Benjamin Himpel}
\affiliation{%
 \institution{School of Informatics, Reutlingen University}
 \city{Reutlingen}
 \country{Germany}}
\email{Benjamin.Himpel@reutlingen-university.de}
\orcid{0000-0002-3737-8379}

\author{Georgios Keramidas}
\affiliation{%
  \institution{School of Informatics, Aristotle University of Thessaloniki}
  \city{Thessaloniki}
  \country{Greece}}
\email{gkeramidas@csd.auth.gr}
\orcid{0000-0003-0460-6061}


\begin{abstract}

Deep neural networks (DNNs) have become indispensable in many real-life applications like natural language processing, and autonomous systems. However, deploying DNNs on resource-constrained devices, e.g., in RISC-V platforms, remains challenging due to the high computational and memory demands of fully connected (FC) layers, which dominate resource consumption. Low-rank factorization (LRF) offers an effective approach to compressing FC layers, but the vast design space of LRF solutions involves complex trade-offs among FLOPs, memory size, inference time, and accuracy, making the LRF process complex and time-consuming.
This paper introduces an end-to-end LRF design space exploration methodology and a specialized design tool for optimizing FC layers on RISC-V processors. Using Tensor Train Decomposition (TTD) offered by TensorFlow T3F library, the proposed work prunes the LRF design space by excluding first, inefficient decomposition shapes and second, solutions with poor inference performance on RISC-V architectures. Compiler optimizations are then applied to enhance custom T3F layer performance, minimizing inference time and boosting computational efficiency. On average, our TT-decomposed layers run 3× faster than IREE and 8× faster than Pluto on the same compressed model. This work provides an efficient solution for deploying DNNs on edge and embedded devices powered by RISC-V architectures.

\end{abstract}


\maketitle
\noindent\textbf{Note:}{This is a pre-print version of the manuscript. The final version is published in ACM Transactions on Embedded Computing Systems, DOI: https://doi.org/10.1145/3768624}

\section{Introduction}\label{introduction}



Deep Neural Networks (DNNs) are integral to modern Artificial Intelligence (AI), powering applications in image processing \cite{image-processing}, natural language processing (NLP) \cite{nlp}, speech recognition \cite{speech-recognition}, and autonomous systems \cite{autonomous-system}. Their ability to learn complex patterns has driven breakthroughs in areas like object detection \cite{object-detection} and language translation \cite{language-translation}, making them indispensable across fields ranging from healthcare \cite{healthcare} to autonomous driving. 
Despite their widespread success, deploying DNNs on resource-constrained devices, such as edge platforms, presents significant challenges~\cite{rcd-challenges}. These models require substantial computational power and memory, leading to high latency, energy consumption, and storage demands, which hinder their feasibility on devices with limited processing capabilities~\cite{rcd-challenges}.

RISC-V, an open-source Instruction Set Architecture (ISA), has become a popular choice for embedded and edge devices~\cite{risc-v-acceleration}. This has driven a growing demand for optimized DNN implementations specifically tailored to the unique constraints and capabilities of RISC-V platforms. 
To address these challenges, DNN compression techniques such as low-rank factorization (LRF), pruning, quantization, and knowledge distillation have become essential tools \cite{compression-techniques}. These methods aim to reduce the size and computational complexity of DNNs, enabling efficient deployment on devices with limited memory, processing power, and energy capacity. By reducing the number of parameters and FLOPs, compression techniques significantly enhance the efficiency and accessibility of DNN models across a wide range of applications \cite{compression-techniques}. Among these methods, LRF \cite{milad2} has gained particular attention for its ability to exploit redundancies in neural network weights. LRF achieves compression by approximating weight matrices or tensors through decomposition into smaller components, effectively reducing both the parameter count and computational cost \cite{milad2}. This approach can result in substantial reductions in memory usage and faster inference times, making it especially suitable for edge devices.

\begin{figure*}[tbp]
\centerline{\includegraphics[trim={1mm 0 0 1}, clip,width=1\textwidth]{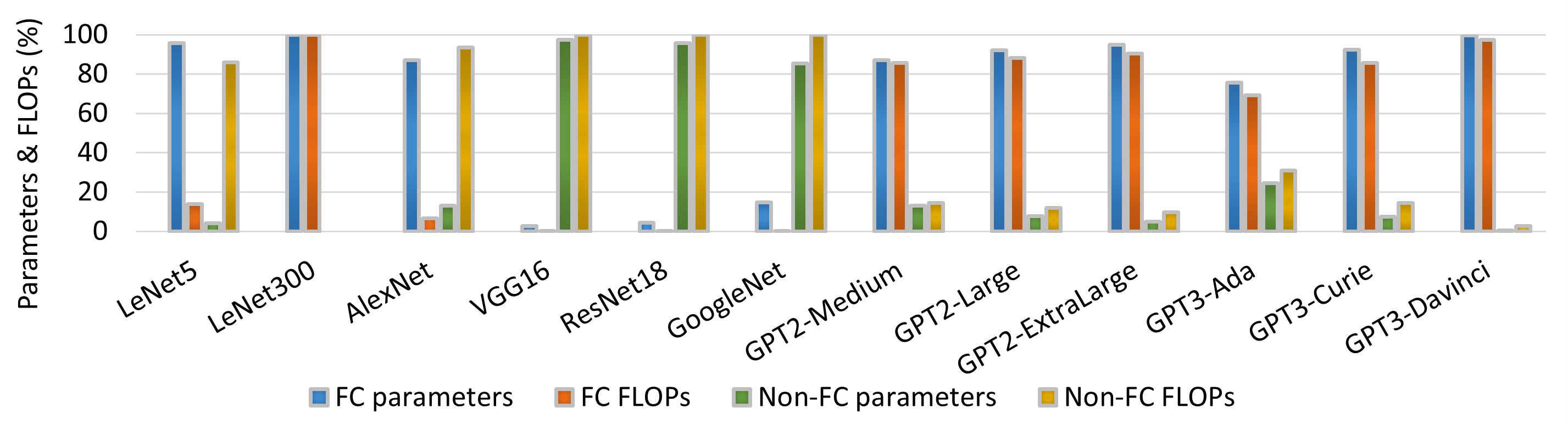}}
\caption{Parameters and FLOPs percentage of FC and non-FC parts for various DNN models}
\label{fig:params_flops}
\end{figure*}

Several algorithms support LRF, including Tensor Train Decomposition (TTD)\cite{Oseledets2011}, Canonical Polyadic (CP) decomposition\cite{cp}, and Singular Value Decomposition (SVD)~\cite{svd}. While SVD is specifically designed for compressing two-dimensional arrays, methods like TTD are optimized for tensors (multi-dimensional arrays). To apply tensor-based methods to matrices, the matrices must first be reshaped into tensors before factorization.
In this work, we utilize the TTD method as implemented in the T3F library~\cite{t3f} for TensorFlow. T3F performs decomposition by breaking down tensors into smaller tensors (called cores) interconnected via the Kronecker product~\cite{t3f}, enabling efficient compression and computation.

Fully connected (FC) layers, especially in large language models (LLMs), present significant challenges due to their high computational and memory demands. As shown in Figure~\ref{fig:params_flops}, FC layers dominate the resource usage in many DNN models, including Convolutional Neural Networks (CNNs) and LLMs. The figure highlights the substantial portion of parameters and floating-point operations (FLOPs) in the FC\footnote{In LLMs, each transformer layer can be decomposed into multiple subcomponents. Specifically, the transformer layer comprises a multi-head self-attention mechanism, which is primarily implemented as a FC layer operating on the input queries, keys, and values, followed by a subsequent FC layer for output processing. Additionally, the transformer layer includes feedforward sublayers, typically implemented as FC layers with activation functions. These core components are interspersed with layer normalization and residual connections to enhance gradient flow and model stability.} versus non-FC parts of well-known models\footnote{We estimated the architecture and layer details of GPT-3 models based on publicly available suggestions since official specifications are not provided.}. This dominance often results in increased latency and energy consumption, making it difficult to deploy complex DNNs efficiently on resource-constrained platforms. To address this, there is a critical need for optimized implementations that balance performance with hardware constraints.



\begin{figure}[tbp]
    \centering
    \begin{minipage}{0.5\linewidth}
        \centering
        \includegraphics[width=\linewidth]{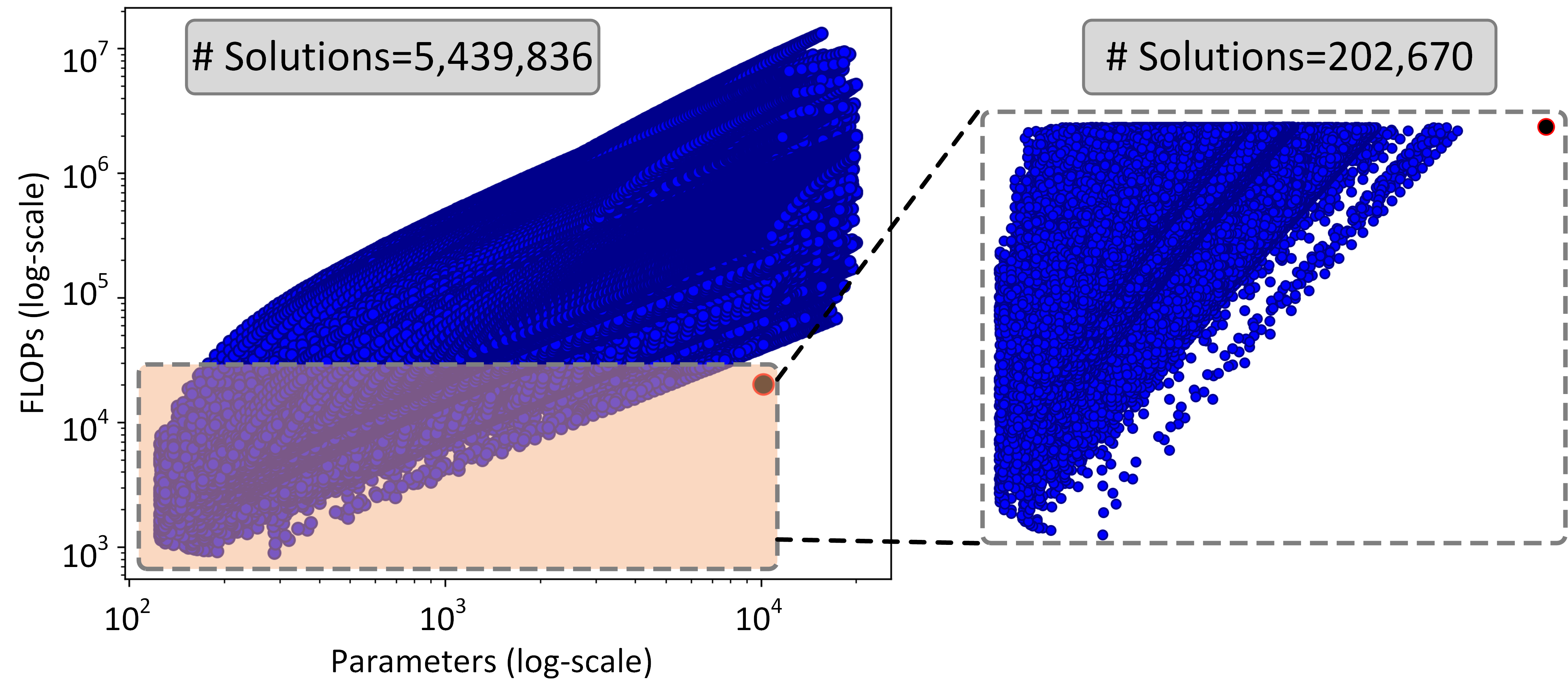}
        \caption*{(a)}
        \label{fig:design_space_motivation}
    \end{minipage}
    \hfill
    \begin{minipage}{0.45\linewidth}
        \centering
        \includegraphics[trim={1mm 0 0 0}, clip,width=\linewidth]{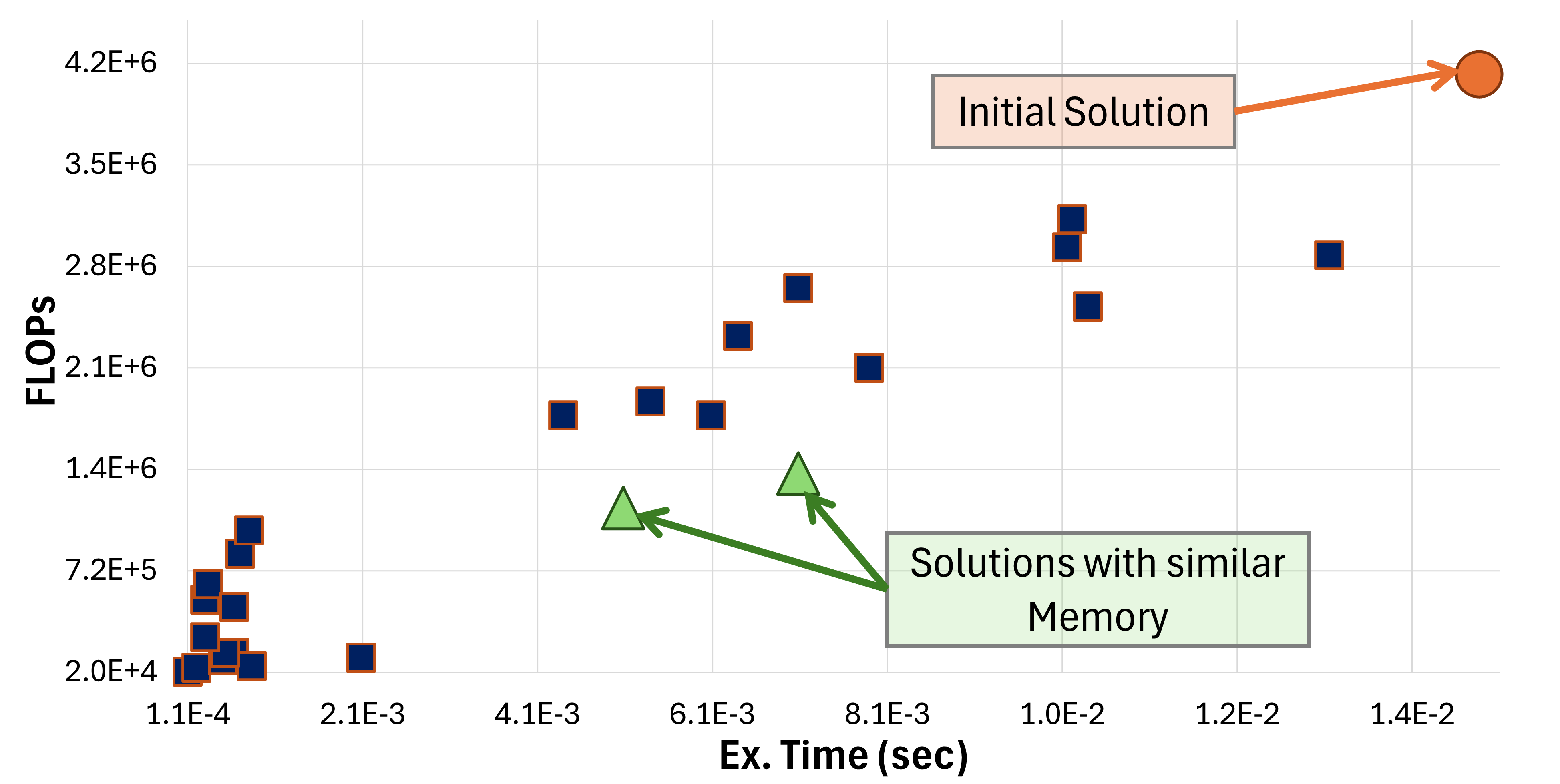}
        \caption*{(b)}
        \label{fig:execution_time_motivation}
    \end{minipage}
    \caption{a) The DS of parameter vs. FLOPs is shown for all solutions (left) and solutions with less parameters and FLOPs than the initial layer (right), for a FC layer with shape of 120x84 based on the T3F library; the big black circle corresponds to the parameters and FLOPs of the initial layer. b) FLOPs vs. inference execution time for different LRF solutions of the studied FC layer}
    \label{fig:motivation}
\end{figure}

Applying LRF to FC layers is challenging due to the vast Design Space (DS). 
Figure~\ref{fig:motivation}.a illustrates the enormous number of potential LRF configurations even for a very small FC layer of size $120 \times 84$. Exploring all possible solutions even for such a small layer is a very time consuming process. We use this small layer as an example in Figure~\ref{fig:motivation}.a because the DS of larger layers can be massive (up to $10^{33}$ possible solutions). 
Efficient navigation of the DS is essential to identify efficient solutions that balance performance and resource usage without exhaustive testing.

Furthermore, as illustrated in Figure~\ref{fig:motivation}.b, LRF solutions with comparable memory footprints can exhibit significant differences in FLOPs and execution time. Notably, FLOPs and execution time do not always align (Figure~\ref{fig:motivation}.b). This disparity highlights the importance of evaluating configurations not only for memory efficiency but also for computational performance, particularly when targeting edge platforms, such as RISC-V. Achieving an optimal balance between speed and efficiency requires both advanced compiler optimization strategies and comprehensive Design Space Exploration (DSE).

To address these challenges, this work proposes a novel LRF DSE methodology and a specialized design tool tailored for optimizing FC layers in DNNs that combines high-level and low-level DS reduction strategies with compiler optimization techniques specifically designed for RISC-V platforms. First, high-level DS reduction is applied to narrow down the DS of LRF solutions by excluding decomposition shapes that do not achieve low FLOPs. This is followed by low-level DS reduction, which refines the search by eliminating solutions that result in inefficient execution times, ensuring that only the most promising configurations are retained. Finally, compiler optimizations are applied to further enhance the performance of the selected configurations on RISC-V processors, focusing on reducing inference times and improving computational efficiency. By integrating these three complementary techniques —high-level DS reduction, low-level DS reduction, and RISC-V-specific compiler optimizations— our approach effectively reduces the design space, enhances execution speed, and enables the efficient deployment of DNNs on RISC-V platforms.

%
%


The proposed DSE methodology operates at the layer level, making it compatible with other model-level optimization techniques. For each FC layer, the methodology generates a list of potential solutions rather than a single one. This flexibility ensures that if the accuracy constraint is not met, alternative solutions can be selected and evaluated from the list.
Our experimental results demonstrate that the proposed methodology reduces the DS by several orders of magnitude. Furthermore, the custom T3F layers achieve an average inference time reduction of three times compared to IREE \cite{iree} and eight times compared to Pluto \cite{pluto}, highlighting the efficiency and effectiveness of our approach.

This research makes several significant contributions: 
\begin{itemize}
    \item It introduces a novel method for efficiently reducing the DS in TTD, which facilitates faster and more effective exploration of factorized solutions

    \item A compiler optimization approach is proposed to accelerate the custom T3F layers of TensorFlow on RISC-V processors

    \item A comprehensive experimental evaluation is conducted on seven CNNs and six LLMs, showcasing that the proposed method not only reduces the exploration space, but also achieves substantial speedups for custom T3F layers
\end{itemize}

The remainder of this paper is organized as follows. Section \ref{sec:background} provides the necessary background information. Section \ref{sec:related-work} reviews relevant literature. Section \ref{sec:methodology} details the proposed methodology, followed by Section \ref{sec:experimental-setup}, which outlines our evaluation framework. Section \ref{sec:results} presents the experimental results.  Finally, Section \ref{sec:conclusion} concludes the paper.







\section{Background}\label{sec:background}
\noindent\textbf{Tensor-Train Decomposition:}
Let a FC layer be given by:
\begin{equation}
y = W x + b,
\end{equation}
with output vector $y\in \R^M$, weight matrix $W \in \R^{M\times N}$, input vector $x\in \R^N$, and bias vector $b\in \R^M$. The computation is performed using the well-known Matrix-Vector Multiplication (MVM) kernel, i.e., $y_j = (\sum^N_{k=1} W_{jk}x_k)+b_j$ . To reduce computational complexity and memory requirements, TTD is leveraged to approximate the FC layer through a sequence of smaller yet structured layers rather than applying it directly.


Following the notation introduced in \cite{Oseledets2011} and \cite{Novikov2015}, $M$ and $N$ are factorized into $d$-dimensional components, where $M = \prod^d_{t=1} m_t$ and $N = \prod^d_{t=1} n_t$; note that $d$ is not a fixed number. By reshaping $W$, $x$, $y$ and $b$, the FC layer is approximated by:
\begin{equation}
\cY(i_1,\ldots,i_d)  = \sum_{j_1=1}^{n_1}\ldots \sum_{j_d=1}^ {n_d} G_1[i_1,j_1]\cdot \ldots \cdot G_d[i_d,j_d] \cX(j_1,\ldots,j_d) 
+ \cB(i_1,\ldots,i_d),\label{eq:tf_Nov}
\end{equation}
where \( G_k[i_k, j_k] \) represents the tensor-train cores of size $r_{k-1}\times r_{k}$, \( \mathcal{B}(i_1, \ldots, i_d) \) is a tensor of shape $n_1\times \ldots \times n_k$ (or $[n_1,\ldots,n_k]$) representing $b$, $\cX$ is the input tensor of shape $m_1\times \ldots\times m_d$ (or $[m_1\times \ldots \times m_d]$) representing $x$, and $\cY$ is the output tensor of shape $n_1\times \ldots\times n_d$ representing $y$. In analogy to \cite{Novikov2015}, the representation given in Equation \eqref{eq:tf_Nov} is referred to as the TTD of the FC layer. The sequence $[r_0,\ldots,r_k]$ is referred to as the TT-ranks of the TTD.

First, the meaning of the general form of TTD, Equation~\eqref{eq:tf_Nov}, in the context of a FC needs to be clarified. Each tensor core $G_l$, $l=1,\ldots,d$, is a set of $r_{l-1}\cdot r_l$ matrices each of dimension $m_l \times n_l$. Each matrix can therefore be indexed by a multi-index $(s_1,s_2) \in r_{l-1} \times r_l$. In Equation~\eqref{eq:tf_Nov} each $G_l[i_l,j_l]$ can be viewed as a $r_{l-1}\times r_l$ matrix with $r_0=r_d=1$. Therefore, $G_1[i_1,j_1]\cdot \ldots \cdot G_d[i_d,j_d]$ contracts to a single number. Each tensor core is represented as a 4-dimensional tensor to capture the factorized input and output dimensions, ensuring structured connectivity within the TTD. Specifically, $G_d[i_d,j_d] = G^{(d)}_{:,i_d,j_d,:}$. Subsequently, the computation is performed as follows:
\begin{equation}
	\cY(i_1,\ldots,i_d) 
 = \sum_{j_1=1}^{n_1}\sum_{s_1=1}^{r_1}G^{(1)}_{1,i_1,j_1,s_1}\dots  \left( \sum_{j_d=1}^{n_d}\sum_{s_d=1}^{r_d}G^{(d)}_{s_{d-1},i_d,j_d,s_d} \cX(j_1,\ldots,j_d)\right)  + \cB(i_1,\ldots,i_d),\label{eq:tf_Nov2}
\end{equation}



To compute the Tensor-Train approximation, the T3F library \cite{t3f} is utilized. T3F employs the simplified Equation~\eqref{eq:tf_Nov2} described above in the process.

Let us give an example. Consider a FC layer from LeNet300 model, of shape $[N, M]$=$[784, 300]$. A valid combination of M and N is $(300=5 \times 5 \times 3 \times 2 \times 2)$ and $(784=2 \times 2 \times 2 \times 7 \times 14)$. Thus, the $m_i$/$n_i$ values are $[m_1,m_2,m_3,m_4,m_5] = [5, 5, 3, 2, 2]$ and $[n_1,n_2,n_3,n_4,n_5] = [2, 2, 2, 7, 14]$. Also assume that the rank list is $[r_0,r_1,r_2,r_3,r_4,r_5] =[1, 10, 10, 10, 10, 1]$. In the rank list, the first and last values are always 1 and the intimidate values are integers specified by the user. In this example, the intermediate ranks are assumed to be 10 \footnote{In principle, the intermediate ranks may vary; however, for simplicity, identical values are used in this example. Hereafter, $R$ will be used in place of the rank list phrase, indicating that the intermediate ranks are all equal to $R$.}.



After specifying the combination shapes ($m_t$/$n_t$ values above) and the rank values, we can use T3F library and in particular the command in Equation~\eqref{eq:random_matrix} to decompose the initial weights array W. 
The T3F library is used to decompose the initial weights array $W$, specifically Equation~\eqref{eq:tf_Nov2} by using the following command: 
\[
\verb|W = t3f.random_matrix([[2, 2, 2, 7, 14], [5, 5, 3, 2, 2]], R=10)|
\label{eq:random_matrix}
\]


The aforementioned command decomposes the $W$ matrix into the following 4-dimensional cores:
\begin{align*}
G^0 = [r_0,n_1,m_1,r_1] =& [1,2,5,10]\\
G^1 = [r_1,n_2,m_2,r_2] =& [10,2,5,10]\\
G^2 = [r_2,n_3,m_3,r_3] =& [10,2,3,10]\\
G^3 = [r_3,n_4,m_4,r_4] =& [10,7,2,10]\\
G^4 = [r_4,n_5,m_5,r_5] =& [10,14,2,1].\\
\end{align*}
The shape of each core $G^{(k)}$ is of size $r_{t-1}\times n_t \times m_t \times r_t$, based on Equation~\eqref{eq:tf_Nov}, Equation~\eqref{eq:tf_Nov2} and \cite{Novikov2015}.
These cores are processed from the bottom to the top, or equivalently from the right to the left in 
Equation~\eqref{eq:tf_Nov} and Equation~\eqref{eq:tf_Nov2}, starting with the core with shape $[r_4,n_5,m_5,r_5]$. The processing involves Einsum~\cite{einsum} operation between the core and the corresponding $\cX$ tensor. 

By utilizing the T3F library for the example provided, the aforementioned FC layer is decomposed into a sequence of $d$ Einsum layers, as illustrated in the Python code in Listing \ref{lst:python}.

T3F introduces a custom layer that utilizes Einsum kernels to perform this approximation and decomposition process.
In principle, the custom layer employs Einsum operations that utilize factorized weight cores and reshaped input to generate the factorized output, using a notation that defines the desired operation. Since the weights in the layer are trainable, the cores, $G^0 - G^4$, are also trainable parameters. This approach ensures that the sum of the core parameters and operations remains small enough, providing improved memory efficiency and reduced FLOPs compared to the initial FC layer.

\begin{lstlisting}[language=Python, frame=single, numbers=left, numbersep=5pt, xleftmargin=1.5em, caption={Example of the decomposition of the FC layer into a sequence of Einsum operations in Python}, captionpos=b, label={lst:python}]
x_re = x.reshape(b_5, n_5, r_5)
Out4 = einsum("r4 n5 m5 r5, b5 n5 r5 -> m5 b5 r4", G_4, x_re)

Output4 = Out4.reshape(b4, n4, r4)
Output3 = einsum("r3 n4 m4 r4, b4 n4 r4 -> m4 b4 r3", G_3, Output4)

Output3 = Output3.reshape(b_3, n_3, r_3)
Output2 = einsum("r2 n3 m_3 r3, b3 n3 r3 -> m3 b3 r2", G_2, Output3)

Output2 = Output2.reshape(b_2, n_2, r_2)
Output1 = einsum("r1 n2 m2 r2, b2 n2 r2 -> m2 b2 r1", G_1, Output2)

Output1 = Output1.reshape(b1, n1, r1)
Output0 = einsum("n1 m1 r1, b1 n1 r1 -> m1 b1", G_0, Output1)

y = flatten(Output0) + b
\end{lstlisting}

Einsum, short for Einstein summation notation, provides a concise way to express tensor operations by explicitly defining how indices are combined. For example, the MVM kernel can be expressed as: 
\[
y = einsum("jk,k->j",W,x)
\]

As noted, the output of an Einsum operation is the input to the next one. To perform the Einsum operation, input must first be reshaped. For instance, in the first Einsum operation, input vector $x\in \R^N$ is reshaped into 3-dimensional tensor of size $[N / (n5*r5), n5, r5]$. In Listing \ref{lst:python}, line 4, $b_5 =  N / (n5*r5)$. The same approach is used for the subsequent layers. In the final stage, the output of the final Einsum is reshaped into a vector (also bias is added), representing the output of the FC layer.

Listing \ref{lst:native} presents the C implementation of the native Einsum operation used in T3F.
A graphical representation of Listing \ref{lst:native} is provided in Figure~\ref{fig:acc}. The figure illustrates how the nested loops traverse the tensors $G$, $Input$, and $Output$. The highlighted red regions illustrate the data access patterns for fixed $(m,b,r)$ loops used to produce a single output element.

In Listing \ref{lst:native}, current rank value $r_{t}$ is shown as $rt$ and the previous rank value $r_{t-1}$ is shown as $rt\_1$. As depicted in Figure~\ref{fig:acc}, for the first Einsum operation,  $rt\_1=1$, which eliminates the need for $k$-loop. Similarly, for the final Einsum operation, $rt=1$, resulting in the absence of $r$-loop.

\begin{lstlisting}[language=C, frame=single, caption={Native C Kernel for $einsum("rnmk,bnk -> mbr", G, Input)$}, label={lst:native}]
for (int m = 0; m < mt; m++)
    for (int b = 0; b < bt; b++) 
        for (int r = 0; r < rt; r++)
            for (int n = 0; n < nt; n++) 
                for (int k = 0; k < rt_1; k++){
          Output[m][b][r] += G[r][n][m][k]*Input[b][n][k];
}
\end{lstlisting}

\begin{figure}
    \centering
    \includegraphics[width=1\linewidth]{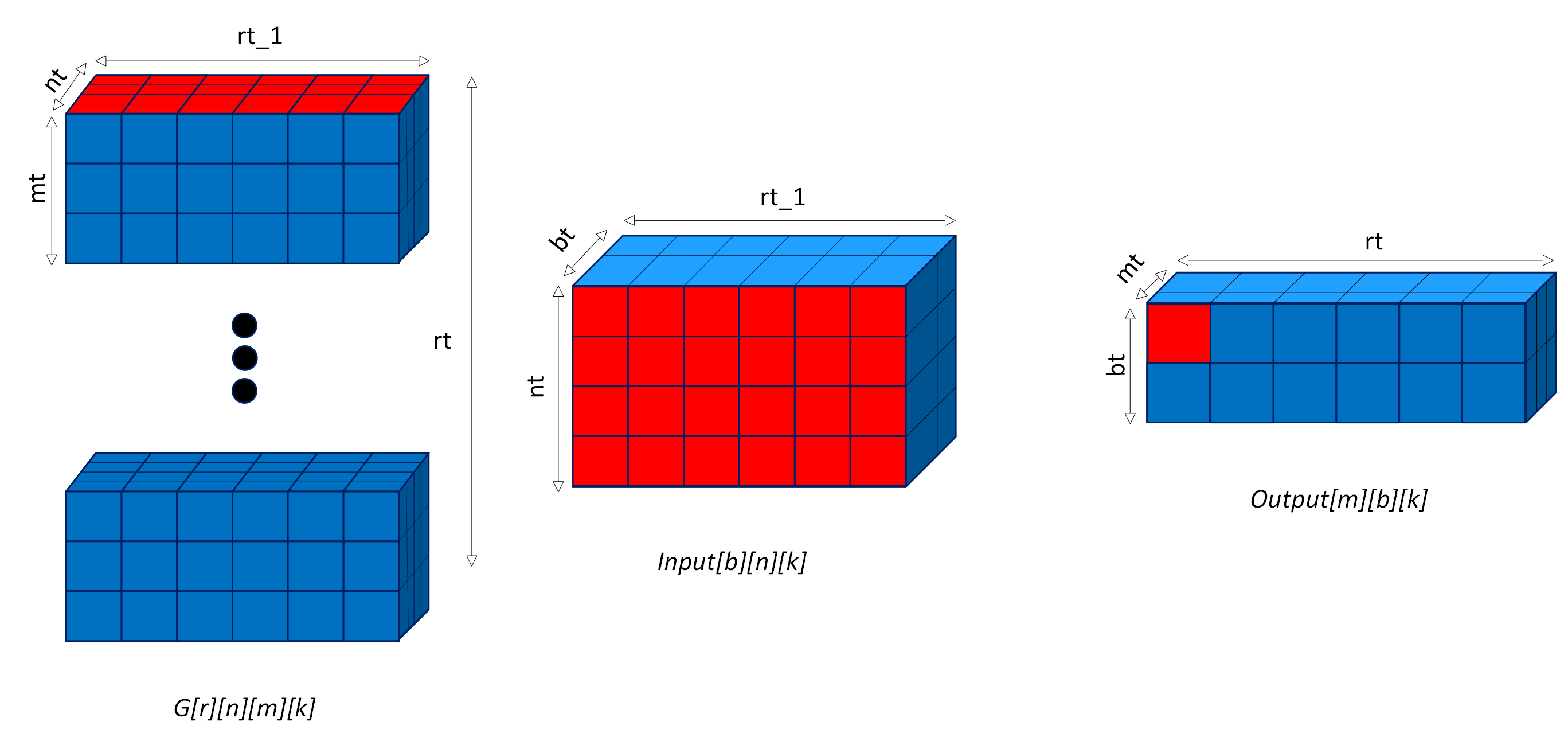}
    \caption{Visual illustration of Listing 2}
    \label{fig:acc}    
\end{figure}


In the following, the equations that determine the number of parameters and FLOPs are outlined. These equations are crucial for selecting the most efficient solutions. 

\noindent\textbf{Computing the number of parameters using TTD:}
The total number of parameters in the unfactorized FC layer is given by $M\times N + M$, where $N \times M$ represents the number of parameters for the weights and $M$ accounts for the bias. 
For the factorized FC layer using TTD, the number of parameters is given by:

\begin{equation} \text{Memory} = M+ \sum_{i=1}^d \left(r_{i-1}  \cdot m_i \cdot n_i \cdot r_i\right),\label{eq:Memory}
\end{equation} 

\noindent where the first term corresponds to the number of parameters for the bias, while the second term accounts for the number of parameters of the weights.
Specifically, there are $d$ cores, $G$ tensors, each of size $r_{t-1}\times n_t \times m_t \times r_t$, contributing to the parameter count.

\noindent\textbf{{Computing FLOPs using TTD:}}
The FLOPs value involves summing the FLOPs across all Einsum operations:
\begin{equation}\label{eq:FLOPs_1}  
\text{FLOPs} = M + \sum_{t=1}^{d} \text{FLOPs}^{(t)},
\end{equation}  
where $M$ accounts for the FLOPs that relate to the bias. In the example presented in Listing \ref{lst:python} there are $d=5$ Einsum operations. The first Einsum operation corresponds to $t=d$ and the last Einsum operation corresponds to $t=1$. Note, the parameter $b_5$ is the result after a reshape layer. Similarly, the parameter $b_1$, essential for calculating the FLOPs in the final Einsum, depends on the parameters of the previous Einsum operations. Therefore, a detailed analysis is required to determine the $b_i$ parameter across each Einsum operation.

\begin{prop}\label{prop:flopsd}
For the first Einsum layer, the FLOPs are given by:
\begin{equation}
\text{FLOPs}^{(d)} = 2 \cdot n_d \cdot r_d \cdot r_{d-1} \cdot m_d \cdot n_1 \cdot \ldots \cdot n_{d-1}.
\label{eq:flops_subterms}
\end{equation}
\end{prop}

\begin{proof}
To derive the FLOPs equation, consider the subterm from Equation~\eqref{eq:tf_Nov2}:
\begin{equation}
S^{(d)}_{i_d,j_1,\ldots,j_{d-1},s_{d-1}} :=
\sum_{j_d=1}^{n_d} \sum_{s_d=1}^{r_d} G^{(d)}_{s_{d-1},i_d,j_d,s_d} \mathcal{X}_{j_1,\ldots,j_d}.
\label{eq:tf_core}
\end{equation}

The subterm \( S^{(d)} \) is independent of \( i_1, \ldots, i_{d-1} \) and \( s_1, \ldots, s_{d-2} \). Thus, it suffices to compute this term for \( i_d = 1, \ldots, m_d \) and for all \( j_1 = 1, \ldots, n_1 \), \( j_2 = 1, \ldots, n_2 \), up to \( j_{d-1} = 1, \ldots, n_{d-1} \), and \( s_{d-1} = 1, \ldots, r_{d-1} \). These computations can then be reused for subsequent steps.

For each of the \( m_d \cdot n_1 \cdot n_2 \cdot \ldots \cdot n_{d-1} \cdot r_{d-1} \) subterms \( S^{(d)}_{i_d,j_1,\ldots,j_{d-1},s_{d-1}} \), the computation in Equation~\eqref{eq:tf_core} involves:
\begin{equation}
S^{(d)}_{i_d,j_1,\ldots,j_{d-1},s_{d-1}} = 
\sum_{j_d=1}^{n_d} \sum_{s_d=1}^{r_d} G^{(d)}_{s_{d-1},i_d,j_d,s_d} \mathcal{X}_{j_1,\ldots,j_d}.
\end{equation}

This requires \( n_d \cdot r_d \) multiplications and \( n_d \cdot r_d \) additions per subterm. Thus, the total FLOPs for all subterms are computed by multiplying the FLOPs per subterm, \( 2 \cdot n_d \cdot r_d \), with the number of subterms, \( r_{d-1} \cdot m_d \cdot n_1 \cdot \ldots \cdot n_{d-1} \). This results in:
\begin{equation}
\text{FLOPs}^{(d)} = 2 \cdot n_d \cdot r_d \cdot r_{d-1} \cdot m_d \cdot n_1 \cdot \ldots \cdot n_{d-1}.
\label{eq:flops_subterms}
\end{equation}
\end{proof}

\begin{rmk}
If \( r_d = 1 \), Equation~\eqref{eq:flops_subterms} simplifies to:
\begin{equation}
\text{FLOPs}^{(d)} = 2 \cdot n_d \cdot r_{d-1} \cdot m_d \cdot n_1 \cdot \ldots \cdot n_{d-1}.
\end{equation}
\end{rmk}

\begin{prop}
The total FLOPs are given by:
\begin{equation}
\text{FLOPs} = M + \sum_{t=1}^{d} 2 \cdot r_t \cdot r_{t-1} \cdot m_t \cdot \ldots \cdot m_d \cdot n_1 \cdot \ldots \cdot n_t.
\label{eq:FLOPs_general}
\end{equation}
\end{prop}

\begin{proof}
Assume that \( S^{(t+1)}_{i_{t+1},\ldots,i_d,j_1,\ldots,j_t,s_t} \) has been computed for all indices for \( t = 1, \ldots, d-1 \). Analogous to \( S^{(d)} \), each of the \( m_t \cdot \ldots \cdot m_d \cdot n_1 \cdot \ldots \cdot n_{t-1} \cdot r_{t-1} \) subterms can be computed as:
\begin{equation}
S^{(t)}_{i_t,\ldots,i_d,j_1,\ldots,j_{t-1},s_{t-1}} := 
\sum_{j_t=1}^{n_t} \sum_{s_t=1}^{r_t} 
G^{(t)}_{s_{t-1},i_t,j_t,s_t} 
S^{(t+1)}_{i_{t+1},\ldots,i_d,j_1,\ldots,j_t,s_t}.
\label{eq:intermediate_einsum}
\end{equation}

For each subterm, this involves \( n_t \cdot r_t \) multiplications and \( n_t \cdot r_t \) additions. The FLOPs for each level \( t \) are therefore:
\begin{equation}
\text{FLOPs}^{(t)} = 2 \cdot r_t \cdot r_{t-1} \cdot m_t \cdot \ldots \cdot m_d \cdot n_1 \cdot \ldots \cdot n_t.
\label{eq:flops_othersubterms}
\end{equation}

Adding the \( M \) FLOPs required for \( \mathcal{B} \), the total FLOPs are:
\begin{equation}
\text{FLOPs} = M + \sum_{t=1}^{d} \text{FLOPs}^{(t)} 
= M + \sum_{t=1}^{d} 2 \cdot r_t \cdot r_{t-1} \cdot m_t \cdot \ldots \cdot m_d \cdot n_1 \cdot \ldots \cdot n_t.\qedhere
\label{eq:FLOPs_general_final}
\end{equation}
\end{proof}

\section{Related Work}\label{sec:related-work}

The challenge of deploying DNNs on resource-constrained devices has led to significant research in model compression and optimization techniques. Among these, LRF is a well-known approach that reduces the size of weight matrices/tensors by approximating them with lower-dimensional components. Several studies have demonstrated the effectiveness of LRF in compressing DNN models, leading to reduced memory usage and lower inference times. 
Various LRF techniques have been proposed to decompose 2-dimensional matrices (FC layers). These include SVD \cite{svd}, QR decomposition \cite{qr}, interpolative decomposition \cite{interpolative}, and non-negative matrix factorization \cite{non-negative}. Since tensors extend matrices to higher dimensions, specialized methods such as Tucker Decomposition \cite{tucker}, CP decomposition \cite{cp}, and TTD \cite{Oseledets2011} are employed for tensor decomposition. An alternative approach involves reshaping tensors into 2-dimensional matrices and applying conventional matrix decomposition methods \cite{reshape-tensor} or vise versa. Recent efforts, such as \cite{t3f}, leverage TTD and tensor contraction to incorporate LRF into 2-dimensional matrices (FC layers), achieving further compression by reducing the number of parameters.

The integration of LRF in DNNs has been widely explored \cite{LRF3, manual1, genetic}. Many studies have adopted a manual rank selection approach \cite{LRF3, LRF4, manual1, manual2}, where the principal challenge is determining the optimal rank for each layer to balance compression rate and model accuracy effectively. Another significant challenge is the computational overhead, as each solution requires retraining or iterative fine-tuning. To mitigate the limitations of manual rank selection, various automated approaches have been proposed. For example, analytical methods such as Variational Bayesian Matrix Factorization (VBMF) \cite{VBMF} and global optimization strategies based on machine learning, such as Bayesian Optimization \cite{bayesopt}, have been applied to automate layer-wise rank selection \cite{VBMF1, VBMF2}. However, these techniques are limited by their focus on individual layer weight tensors, neglecting important considerations: i) constraints imposed by the target device and ii) potential for code-level optimizations.



There has also been significant interest in DSE for DNN optimization. DSE methods aim to navigate the vast configuration space of combination shape, rank list, architectural choices, and hardware settings to identify the optimal trade-offs between performance, accuracy, and resource usage. 
Previous works such as \cite{milad1, milad2, milad3} address the challenge of large search space and rank selection in FC layers through a DSE methodology. However, their solutions are limited to CNN models and do not consider large models with vast FC layers since the DS is huge. 
Unlike prior studies focusing on CNNs, \cite{vision-transformer} applies an iterative, greedy selection metric to determine layer ranks in a transformer-based model (Vision Transformer). However, this approach remains specific to a single model type. Moreover, these approaches do not take into account the target device and code-level optimizations, which are critical for efficient deployment on real hardware.

Beyond compression, several works have focused on code optimization strategies to enhance the execution efficiency of DNNs on various hardware platforms. One main goal of this work is to optimize the T3F library for a RISC-V architectures. Manually optimized libraries such as OneDNN \cite{onednn}, cuDNN \cite{cudnn}, and BLIS \cite{blis} are meticulously fine-tuned to harness the full potential of the underlying hardware, with cuDNN specifically optimized for NVIDIA GPUs, while OneDNN and BLIS provide extensive optimizations, including support for the RISC-V architecture. Currently, no vendor optimized library supports the T3F library, as its primary objective is to facilitate the process of LRF rather than optimizing the execution time of the process.


Given the vast exploration space and the current limitations of compilers and vendor libraries, semi-automatic optimization approaches, e.g., Pluto \cite{pluto}, PoCC \cite{pocc}, PPCG \cite{ppcg}, and Tiramisu \cite{tiram}, have emerged to optimize affine loop kernels. These tools leverage the polyhedral model for dependency analysis and transformations, enabling efficient parallelization and locality optimizations.

In recent years, several graph-based or deep learning (DL) compilers have been developed, including GLOW \cite{glow}, IREE \cite{iree}, and TFLM \cite{tflm}. These DL compilers are designed to convert high-level, typically graph-based, descriptions of DNNs into executable code that can run on target hardware platforms, such as CPUs, GPUs, or accelerators. This process involves a sequence of optimizations aimed at minimizing arithmetic operations and memory usage, thereby enhancing efficiency.

For the first time, this research introduces a novel and holistic DSE methodology that integrates LRF compression with tailored compiler optimization strategies specifically designed for RISC-V platforms. Unlike existing approaches, which often overlook the target platform or fail to incorporate compiler-level optimizations, the proposed methodology simultaneously reduces the DS and optimizes the T3F library's inference execution time on RISC-V architectures. In contrast to prior works such as \cite{milad1, milad2}, which neglect platform-tailored optimizations, and \cite{milad3}, which requires the  execution of multiple LRF kernels on the target platform, our approach leverages advanced compiler optimization strategies to further enhance loop transformations, memory efficiency, and parallelization. 



\section{Proposed methodology}\label{sec:methodology}

This paper introduces an LRF DSE methodology and a specialized design tool tailored for FC layers in DNNs. The primary objectives are to efficiently factorize FC layers using the TTD method and to generate optimized source code for RISC-V-based platforms to achieve minimal inference times. The methodology takes as input the shape of the target FC layer and the RISC-V hardware details. This methodology can be extended to other processor families, too.

It is important to emphasize that generating all possible LRF solutions and calculating their corresponding FLOPs and parameter values is not practical. Constructing the parameter vs. FLOPs DS involves two key steps: (i) extracting and generating all possible combination shapes, and (ii) for each combination shape, generating all LRF solutions using the full range of rank values from the rank list. Given the infeasibility of this exhaustive approach, an LRF methodology is proposed, shown in Figure~\ref{fig:methodology}, that comprises of three key steps.


First, the DS is pruned by considering the FLOPs/memory efficiency trade-offs. A new  method is proposed that prunes the TTD combination shapes that exhibit high FLOPs and memory demands. The most important part of the first step is that it is performed without calculating the exact memory and FLOPs values.
Second, the DS is pruned by focusing on inference efficiency. Using a set of heuristics, solutions that fail to achieve low inference times are identified and are excluded from the DS to streamline the process.
Last, a methodology is proposed for efficiently applying compiler optimizations to the custom T3F layers. This final step involves optimizing the remaining TTD solutions using compiler optimization techniques to ensure peak performance.


\begin{figure*}[htbp]
\centerline{\includegraphics[width=1\textwidth]{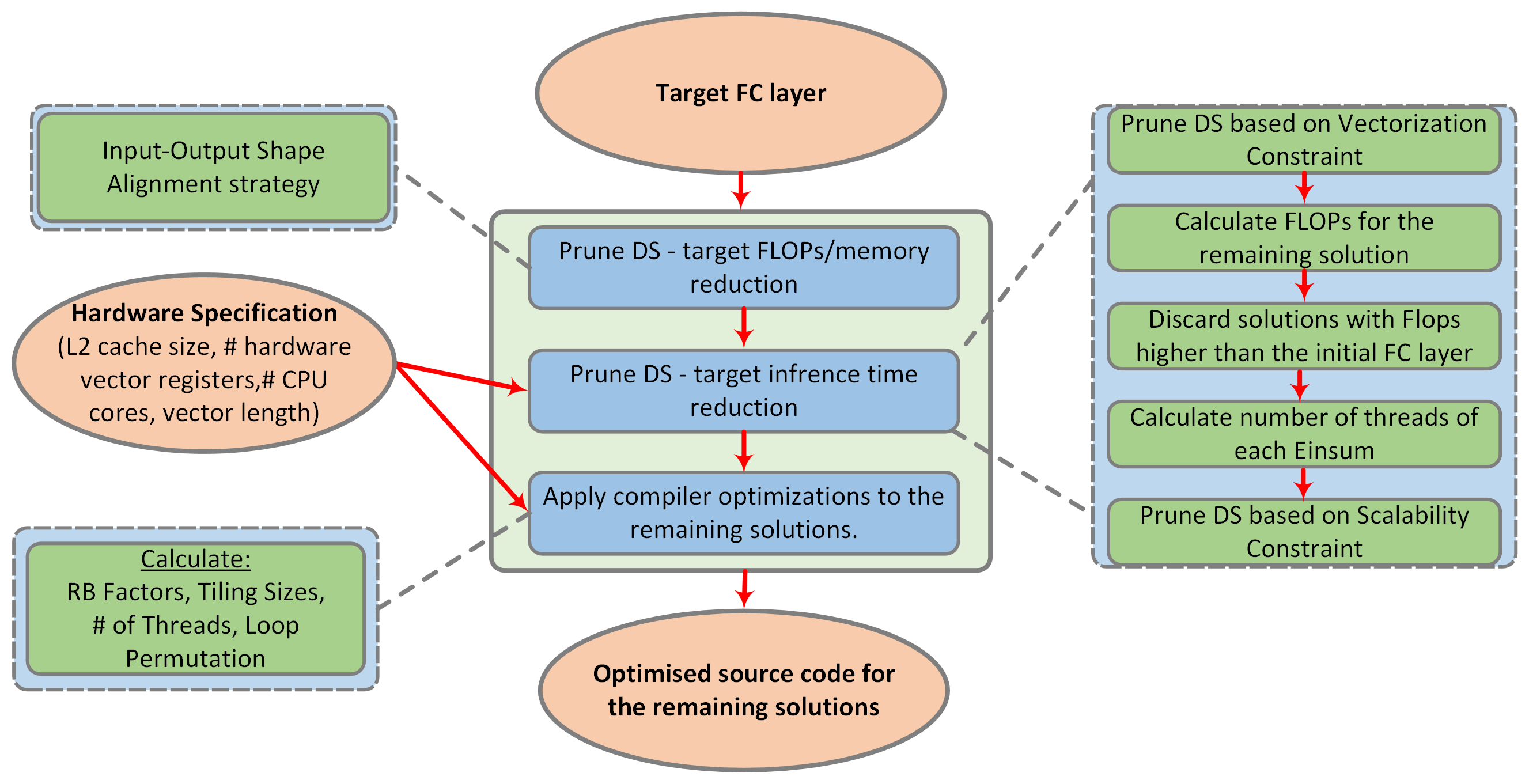}}
\caption{The proposed DSE methodology}
\label{fig:methodology}
\end{figure*}


The proposed layer-wise methodology is designed to be integrated seamlessly with any model-wise approach. For models containing multiple FC layers, the model-wise methodology can apply the proposed approach individually to each layer. Instead of producing a single solution per FC layer, the proposed methodology generates a list of potential solutions. This flexibility allows for adjustments if the accuracy constraint is not initially met, enabling the selection of alternative solutions from the list to achieve the desired accuracy while maintaining efficiency.

LRF is a well-established technique, and prior studies have shown that the accuracy degradation introduced by rank reduction can be effectively addressed through fine-tuning or retraining from scratch~\cite{VBMF1, vision-transformer}. In our previous works~\cite{milad1, milad2, milad3}, we explored strategies for balancing compression ratio and accuracy preservation, including post-factorization fine-tuning and rank selection methods.
The methodology presented in this paper builds upon those findings, but a detailed accuracy analysis lies beyond the scope of the current study and is left for future work.


\subsection{Prune the DS based on FLOPs vs. memory trade-offs}

As discussed in Section 2, each LRF solution generated using the T3F library relies on a combination set that reshapes the original matrix and an associated list of rank values. To ensure no efficient\footnote{In terms of memory and/or FLOPs} solutions are excluded from the DS, it is essential to explore all feasible reshaping combinations for the original matrix and determine upper bounds for the rank values. These upper bounds are specific to each combination shape and must be extracted individually.
However, this process is extremely time-consuming, even for small layers, as it requires calculating FLOPs and parameters for each combination shape and each set of rank values. This is why a new approach is proposed in this work that reduces the DS without calculating these values.

The objective of this subsection is to identify the permutation with the minimum FLOPs among all possible permutations for each combination shape. Notably, each combination shape can have multiple permutations, each with distinct FLOP requirements. 
To streamline this process and focus on the most efficient solutions, we propose an input-output shape alignment strategy. This approach reduces the number of configurations to evaluate and effectively prunes the DS, ensuring only solutions with the lowest FLOPs are retained for each parameter level.

\begin{defn}\label{def:aligned}
The combination shapes are referred to as {\em aligned} if \( n_1 \leq n_2 \leq \ldots \leq n_d \) and \( m_1 \geq m_2 \geq \ldots \geq m_d \). We also say that the input shape is {\em aligned} with the output shape.
\end{defn}

The rationale behind aligning the combination shapes is explained in the following proposition.

\begin{prop}\label{prop:strategy}
The dimensions \( m_s \) and \( n_s \) appear as factors in the summands of FLOPs given by Equation~\eqref{eq:FLOPs_general} exactly \( s \) and \( d-s+1 \) times, respectively.
\end{prop}

\begin{proof}
For a fixed index \( s \), the term \( m_s \) (representing a dimension of the output tensor) appears in \( s \) summands \( \text{FLOPs}^{(t)} \) from Equation~\eqref{eq:flops_othersubterms} for \( t = 1, \ldots, s \). This is because \( m_s \) contributes to the product \( m_t \cdot \ldots \cdot m_d \) whenever \( t \leq s \). In other words, \( m_s \) is included in every summand \( \text{FLOPs}^{(t)} \) starting from \( t = 1 \) up to \( t = s \).

Conversely, the term \( n_s \) (representing a dimension of the input tensor) appears in \( d-s+1 \) summands \( \text{FLOPs}^{(t)} \) for \( t = s, \ldots, d \). This is because \( n_s \) contributes to the product \( n_1 \cdot \ldots \cdot n_t \), which is part of the summand whenever \( t \geq s \). Thus, \( n_s \) is included in every summand \( \text{FLOPs}^{(t)} \) starting from \( t = s \) up to \( t = d \).
\end{proof}

Proposition~\ref{prop:strategy} suggests that a good strategy is to ensure that \( m_s \) decreases for larger \( s \), while \( n_s \) increases for larger \( s \). This is further emphasized by examining the interaction between \( m_s \) and \( n_u \) across summands for $s=t,\ldots,d$ and $u=1,\dots,t-1$. Specifically, \( n_u \) for smaller indices \( u \) is multiplied by \( m_s \) for larger indices \( s \) in the form:
\begin{equation}
n_1 \cdot \ldots \cdot n_{t-1} \cdot m_t \cdot \ldots \cdot m_d.
\end{equation}
This ensures that smaller values of \( n_u \) are multiplied with smaller values of \( m_s \), thereby reducing the overall FLOPs. To minimize the total FLOPs, we propose aligning the input and output shapes as described in Definition~\ref{def:aligned}.




While this heuristic effectively reduces the number of FLOPs, it is equally important to evaluate its impact on the number of parameters. 
Figure~\ref{fig:FLOPsMemory1} and Figure~\ref{fig:FLOPsMemory2} illustrate FLOPs and memory footprints for all shape permutations across six configurations. The proposed aligned shape permutation is highlighted in red. As mentioned, R=4 means that $[r_0,\ldots,r_d]=[1,4,\ldots,4,1]$. 
As it is shown in these graphs, the proposed aligned shape always achieves the minimum FLOPs value and a memory value that is close to the minimum.

\begin{figure}[ht]
	\includegraphics[clip,width=\textwidth,trim={80 150 300 340}]{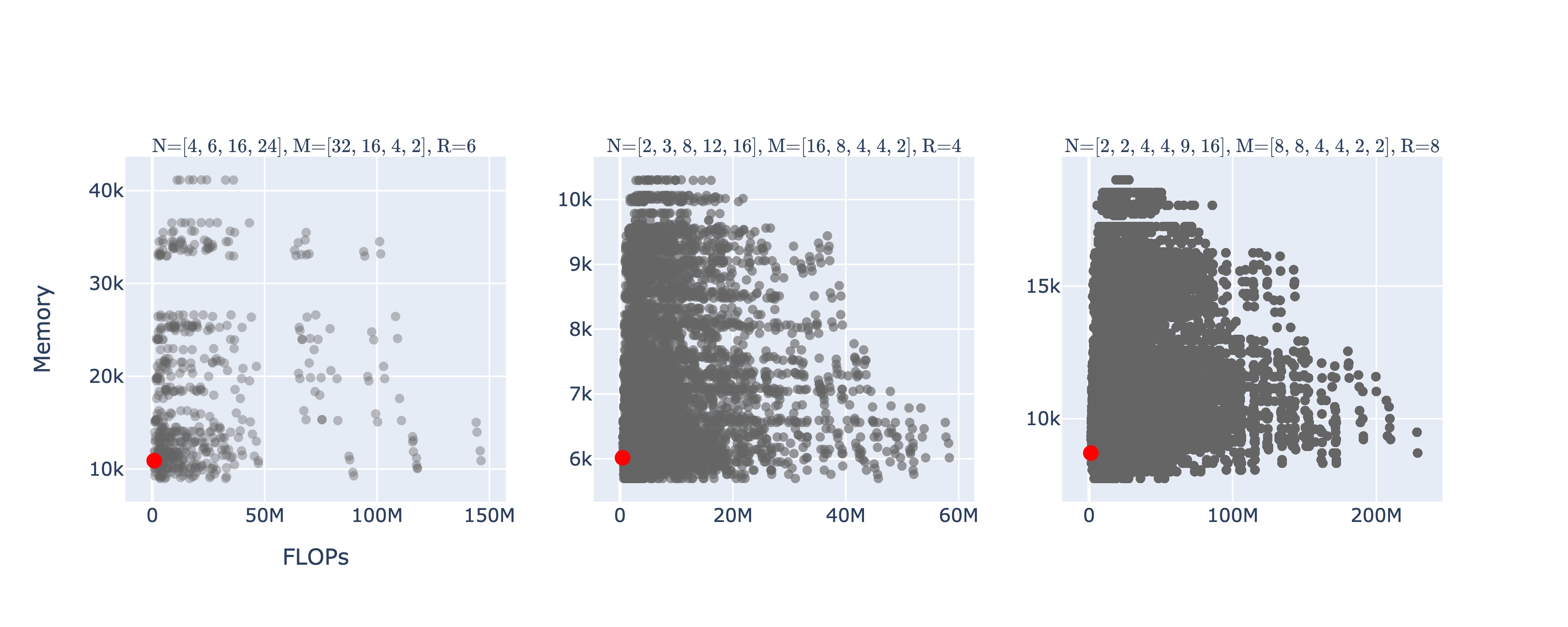}
	\caption{Scatterplots illustrating FLOPs and memory for all shape permutations across three distinct configurations for a CNN layer with $(M,N) = (9216,4096)$\label{fig:FLOPsMemory1}. The proposed aligned shape permutation is highlighted in red}
\end{figure}

\begin{figure}[ht]
	\includegraphics[clip,width=\textwidth,trim={80 150 300 340}]{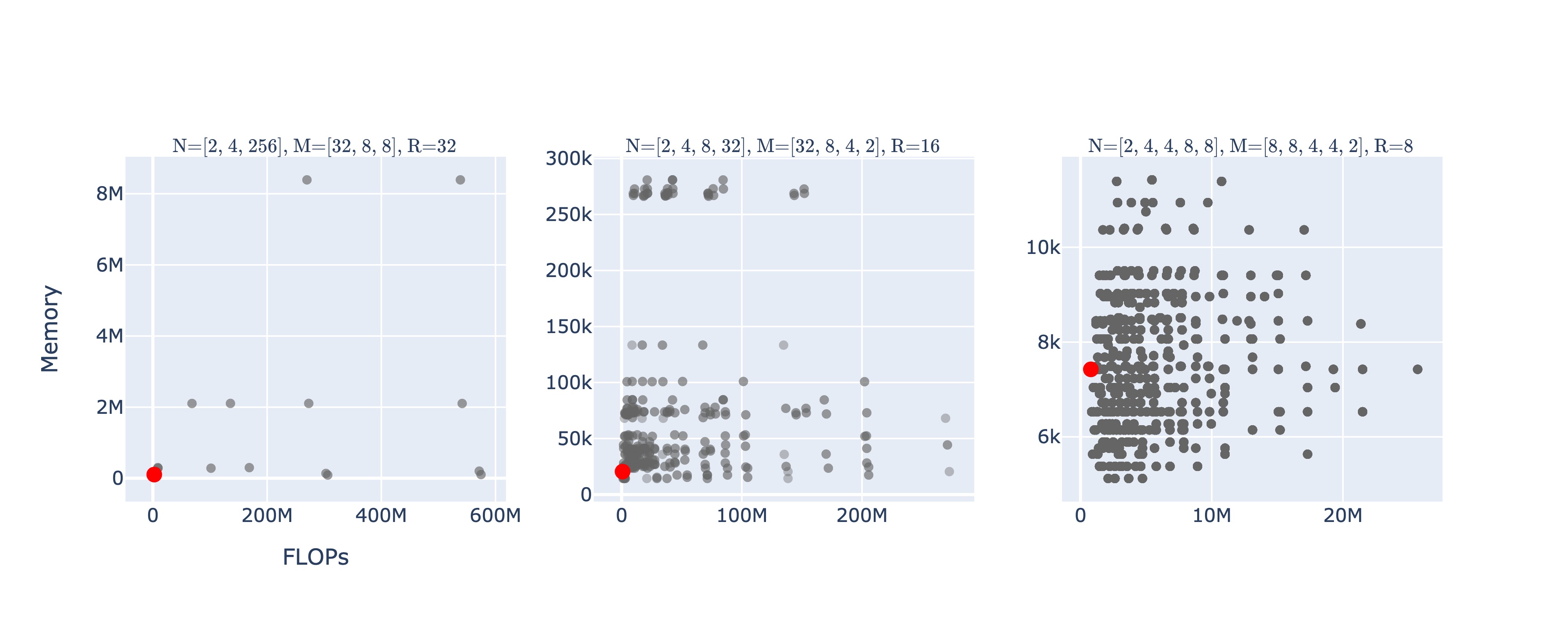}
	\caption{Scatterplots illustrating FLOPs and Memory for all shape permutations across three distinct configurations for a LLM layer with $(M,N) = (2048,2048)$\label{fig:FLOPsMemory2}. The proposed aligned shape permutation is highlighted in red}
\end{figure}


To evaluate the effectiveness of the proposed input-output shape alignment strategy, we assess how well the aligned combination shape performs compared to the maximal and minimal FLOPs/memory values across all permutations \(\sigma\) of the aligned combination shape. Specifically, we compute the following ratios:

\begin{equation}\label{eq:ratio_flops}
\text{ratio}_\text{FLOPs} = \frac{\max_{\sigma}(\text{FLOPs}) - \text{FLOPs}_\text{aligned}}{\max_{\sigma}(\text{FLOPs}) - \min_{\sigma}(\text{FLOPs})},
\end{equation}
and
\begin{equation}\label{eq:ratio_memory}
\text{ratio}_\text{Memory} = \frac{\max_{\sigma}(\text{Memory}) - \text{Memory}_\text{aligned}}{\max_{\sigma}(\text{Memory}) - \min_{\sigma}(\text{Memory})}.
\end{equation}


To evaluate the impact of shape alignment and rank selection in TTD, we constructed a benchmark comprising 374,256 valid TTD configurations. We selected all the studied FC layer shapes (see Tables~\ref{tab:CNNmodels} and~\ref{tab:LLMmodels}); these layers are taken from common CNNs and transformer models. For each of these layers, we generated all compatible aligned input and output tensor shapes as defined in Definition~\ref{def:aligned}, ensuring that the original layer dimensions were preserved while enabling TTD. Each aligned shape pair was then combined with a range of TTD ranks, spanning from 1 to 3064\footnote{The maximum possible rank is inherently constrained by the dimensionality of the input and output tensors, and therefore varies depending on their respective shapes} and increasing primarily in steps of 8.

Every configuration corresponds to a unique combination of input shape, output shape, and TTD rank. The efficiency of alignment is assessed in Figure~\ref{fig:FLOPsMemoryBox} using normalized FLOPs and memory ratios (Equations~\eqref{eq:ratio_flops} and~\eqref{eq:ratio_memory}), which compare each aligned configuration to all valid permutations of input and output shapes. Ratios close to 1 indicate near-optimal compression and computational performance.

Figure~\ref{fig:FLOPsMemory1} and Figure~\ref{fig:FLOPsMemory2} illustrate that the aligned shape permutation consistently minimizes FLOPs, while this is not always the case for memory. The boxplots in Figure~\ref{fig:FLOPsMemoryBox} confirm that the proposed aligned combination shape is optimal with respect to the number of FLOPs, as evidenced by the boxplot for the FLOPs ratio. In fact, the FLOPs ratio boxplot, including whiskers and outliers, reduces to a single line at $1.0$, indicating that the aligned shape permutation is always optimal in terms of FLOPs.

For the memory ratio, the thin box starting at 1 indicates that the data is highly concentrated. Additionally, our computations reveal that about 30\% of all aligned shapes have \(\text{ratio}_\text{Memory}=1\), that is, they are optimal with respect to memory. The box plot demonstrates that, while the aligned combination shape does not guarantee memory optimality in every case, it is almost optimal in the vast majority of instances.

\begin{figure}[ht]
\includegraphics[scale=0.4]{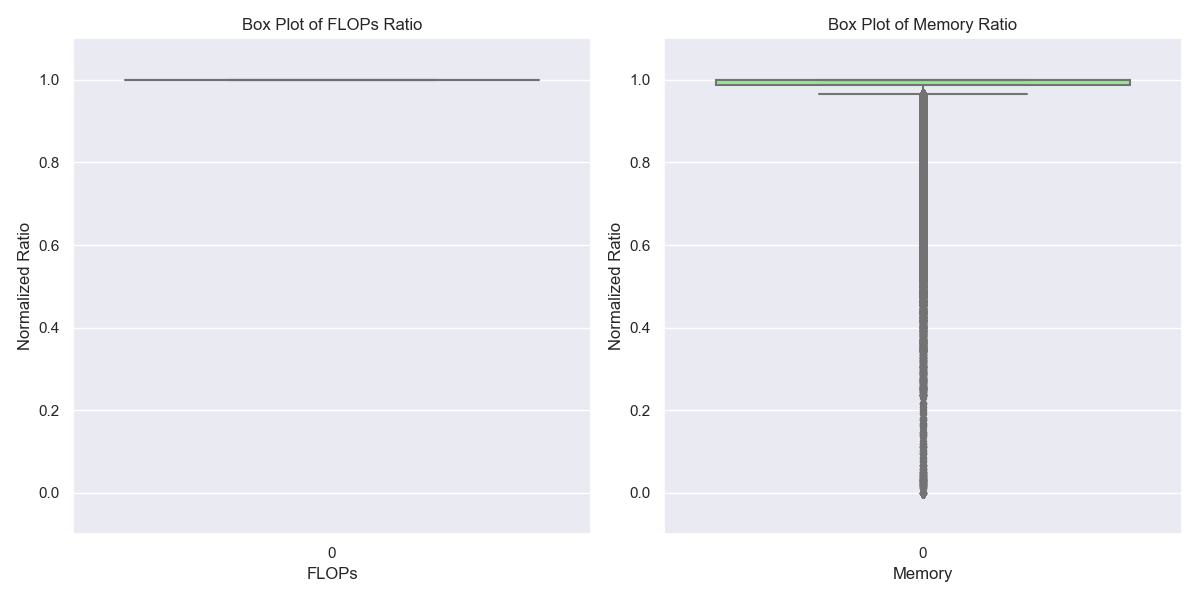}
\caption{Boxplots for FLOPs and nemory ratios compared to the minimal and maximal values across all permutations of the 374,256 cases considered in this study. A normalized ratio of 1 and 0 corresponds to a minimal value and maximal value, respectively\label{fig:FLOPsMemoryBox}}
\end{figure}

The effectiveness of this technique in terms of DS is evaluated by calculating the number of permutations for each configuration. This is formalized in the following proposition as a combinatorial exercise:

\begin{prop}
    Consider aligned combination shapes \([m_1, \ldots, m_d]\) and \([n_1, \ldots, n_d]\). Let \(j\) denote the total number of unique values in \([m_1, \ldots, m_d]\) and \([n_1, \ldots, n_d]\), and let the multiplicities of these unique values be \(k_1, k_2, \ldots, k_j\). Then the total number of permutations of \([m_1, \ldots, m_d]\) and \([n_1, \ldots, n_d]\) is given by:
    \[
    \frac{(d!)^2}{k_1! k_2! \ldots k_j!}.
    \]
    \label{prop:permutations}
\end{prop}

By choosing the aligned combination, where \(m_1 \geq m_2 \geq \ldots \geq m_d\) and \(n_1 \leq n_2 \leq \ldots \leq n_d\), the DS is reduced by a factor of \((d!)^{-2}\) if all \(m_i\) and \(n_i\) values are distinct. When some of the values in \(m_i\) or \(n_i\) are identical, Proposition~\ref{prop:permutations} shows that the number of permutations is reduced by the factor:
\[
\frac{k_1! k_2! \ldots k_j!}{(d!)^2},
\]
where \(j\) is the number of unique values in \([m_1, \ldots, m_d]\) and \([n_1, \ldots, n_d]\), and \(k_1, k_2, \ldots, k_j\) are their respective multiplicities. This accounts for the symmetries arising from identical values.

For example, consider \(d = 5\) and the aligned shapes \([m_1, m_2, m_3, m_4, m_5] = [5, 5, 3, 2, 2]\) and \([n_1, n_2, n_3, n_4, n_5] = [2, 2, 2, 7, 14]\). In the output shape, \(m_1 = m_2 = 5\) with multiplicity \(k_1 = 2\), and \(m_4 = m_5 = 2\) with multiplicity \(k_2 = 2\). In the input shape, \(n_1 = n_2 = n_3 = 2\) with multiplicity \(k_3 = 3\). Other dimensions, each with multiplicity 1, can effectively be ignored. Therefore, the total number of permutations in this case is:
\[
\frac{(d!)^2}{k_1! k_2! k_3!} = \frac{(5!)^2}{2! \cdot 2! \cdot 3!} = 600.
\]
Thus, choosing the aligned combination reduces the DS by a factor of \(\frac{1}{600}\).

Next, we examine how the memory for aligned combinations compares to the minimal and maximal memory values in absolute terms. Figure~\ref{fig:MinMemory} relates the memory for aligned combinations together with ranks to the minimal and maximal memory across all permutations. For example, for aligned shapes \([m_1,m_2,m_3,m_4]=[10, 10, 5, 2]\), \([n_1,n_2,n_3,n_4]=[2, 8, 8, 32]\), and \([r_0,r_1,r_2,r_3,r_4]=[1, 8, 8, 8, 1]\) the memory is \(9352\), while the maximal memory across all permutations is \(26952\), and the minimal memory across all permutations is   \(5224\). This is represented by a blue dot \((x,y)=(9352,26952)\) in the left scatterplot and a red dot \((x,y)=(9352,5224)\) in the right scatterplot of Figure~\ref{fig:MinMemory}.	The left scatterplot shows that the memory for the aligned combination is significantly smaller than the maximal memory across all permutations. However, it is only noticeably larger than the minimal Memory for relatively small memory values.

\begin{figure}[ht]
	\includegraphics[scale=0.45]{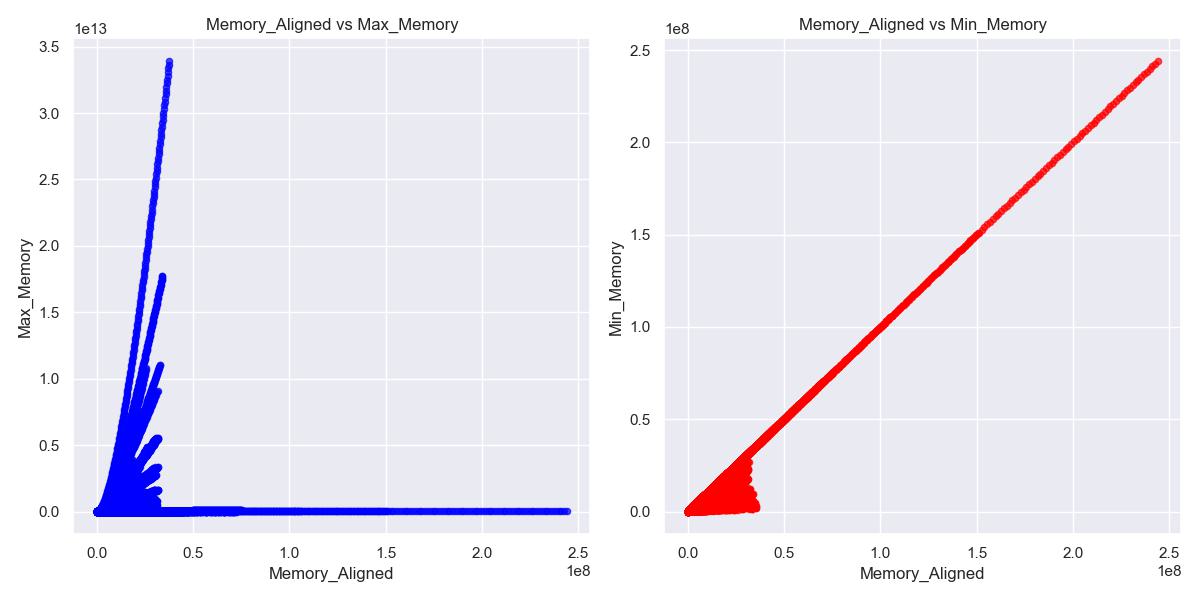}
	\caption{Scatterplots of the maximal (left) and minimal (right) memory across all permutations relative to the memory for the aligned permutation\label{fig:MinMemory}}
\end{figure}

\subsection{Prune the DS based on inference time}




In this subsection, we address the fact that not all solutions with low FLOPs or memory usage necessarily result in efficient inference times. For instance, solutions that do not scale well on multi-core CPUs or those that cannot be effectively vectorized are unlikely to achieve low inference times. To overcome theses challenges, we employ the following steps.


\subsubsection{Prune the DS based on the vectorization constraint}\label{subsec:vec_con}

The vectorization constraint serves two main purposes. First, it ensures that the target vectorized loop is sufficiently large to fully utilize the vector processing capabilities of the underlying hardware. For example, if a loop has only six iterations and the hardware's vector engine can process eight elements in parallel, the vector engine will be underutilized. Second, in cases where the number of loop iterations is not a multiple of the number of elements the engine can process in parallel, the overhead introduced by padding code can be significant. For instance, if a loop has 12 iterations and the vector engine can process eight elements in parallel, the padding overhead will be high. However, if the loop has 100 iterations, the padding overhead becomes negligible.

The selection of loops to be vectorized is discussed in Subsection 4.3. Based on this analysis, the loops selected for vectorization are those corresponding to the TTD rank values, specifically the $rt$ and $rt\_1$ loops in Listing \ref{lst:native}. Two loops, rather than one, are vectorized because the TTD method decomposes the FC layer into multiple Einsum layers, some with a varying number of loops. Consequently, different loops are vectorized by different Einsum layers.

In this paper, we restrict the rank values to be multiples of the vectorization engine's capacity (e.g., multiples of eight), discarding all other values. While high rank values that are not multiples of eight may not incur significant overhead, this work does not implement padding code to address such cases.



\subsubsection{Discard solutions with FLOPs higher than the initial FC layer}

For the remaining solutions, we calculate their FLOPs and memory requirements. Solutions that do not provide FLOPs or memory values lower than the initial, unfactorized solution are deemed inefficient and discarded. This step was not applied earlier because processing the FLOPs and memory values for all potential solutions would have been computationally prohibitive due to the large number of solutions in the initial stages. However, at this stage the number of remaining solutions is significantly lower, thus the overhead of this step is minimal.


\subsubsection{Prune the DS based on the scalability constraint}

In this subsection, the solutions that do not scale well on multi-core CPUs are discarded. This is realized in two different steps. First, for each Einsum kernel in a generated solution (specific combination shape and rank values), the optimal number of threads for efficient execution is selected, and all the other thread combinations are pruned as inefficient. Second, the solutions with high configuration length are considered to exhibit low scalability, thus are also eliminated.

Let us explain the first step in more details. 
The optimal thread count for each Einsum is determined based on its workload. To this end, an experimental analysis was conducted to identify the best thread configuration for maximizing performance on RISC-V platforms.

Figure~\ref{fig:Multy} presents a performance evaluation using varying numbers of threads on a SpacemiT K1 8-core RISC-V CPU. The x-axis represents the number of FLOPs, while the y-axis indicates the speedups to the single-thread execution case (leftmost bar). Note that only one of the two clusters was utilized, thus limiting the evaluation to four cores. As expected, when the workload per thread is sufficiently high, the code scales well. Conversely, when the workload per thread is low, thread creation, and synchronization overheads degrade performance.


Our experimental analysis revealed the following trends for FLOPs values in Einsum kernels. For the Einsum loop kernels with FLOPs value  
\begin{itemize}
    \item lower than $2 \times 10^6$, single-thread execution is optimal
    \item between $2 \times 10^6$ to $4 \times 10^6$, using two threads provide better performance
    \item between $4 \times 10^6$ and $8 \times 10^6$, three threads yield the best results
    \item larger than $8 \times 10^6$ FLOPs, four threads deliver the highest performance
\end{itemize}

Based on these findings, the number of threads for each Einsum kernel is calculated accordingly. It is worth noting that different solutions with the same number of FLOPs include different configuration lengths and as a consequence different number of layers; fewer Einsum layers typically have a higher workload per thread, allowing for better scalability and enabling the assignment of more threads. Conversely, solutions with many Einsum layers distribute the workload across more loop kernels, potentially requiring fewer threads per kernel.

\begin{figure}
    \centering
    \includegraphics[width=1\linewidth]{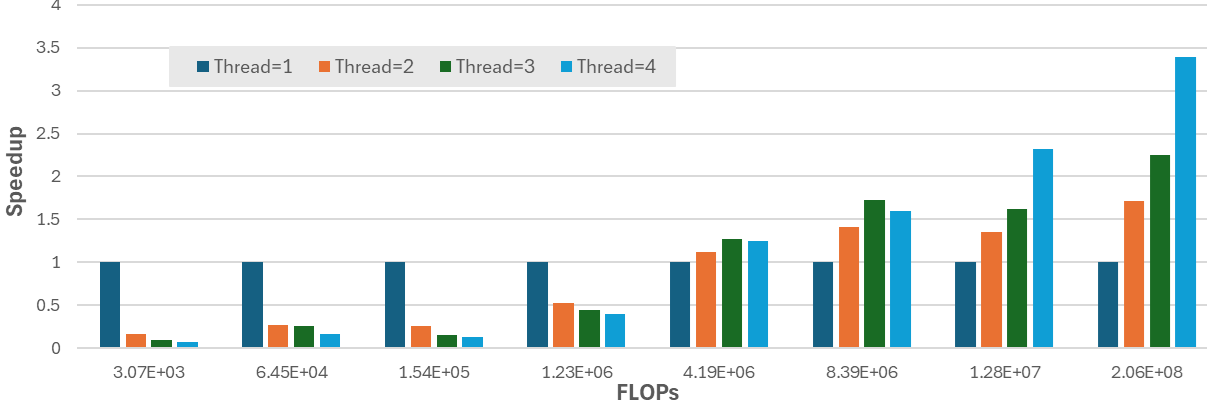}
    \caption{Speedups for different FLOPs and thread values}
    \label{fig:Multy}
\end{figure}


Therefore, the DS is further reduced by pruning high-configuration-length solutions with low scalability. In this step, solutions with combination lengths exceeding four ($d>5$) and poor scalability are discarded. 
Poor scalability is defined as solutions where the heaviest layer has a FLOPs value (\(\text{max\_FLOPs} < 8 \times 10^6\)),  indicating an imbalance in workload distribution.

The reasons that these solutions are discarded are explained hereafter. 
First, pruning these solutions ensures that the remaining exploration space remains sufficiently large to include viable options. Notably, for high-rank values, most configuration shapes naturally have shorter lengths.
Second, they generally result in the same or higher FLOPs compared to configurations with shorter lengths (as illustrated in Figure~\ref{fig:flops_d}, further explained below).


To shed light on the above observation, Figure~\ref{fig:flops_d} is presented. In this figure, we analyze the largest FC layer of AlexNet and calculate the FLOPs for all possible solutions with a fixed rank value of \(8\). As shown in Figure~\ref{fig:flops_d}, the y-axis represents the FLOPs, while the x-axis denotes different solutions across a range of combination lengths (2--12). The results demonstrate that solutions with combination lengths greater than four do not yield significant reductions in FLOPs. This trend is consistent across most models when the rank value is small. Additionally, solutions with high rank values tend to yield only a limited number of high-length combinations. It is also noteworthy that solutions with low combination lengths, and consequently fewer Einsum layers, typically have a higher workload per thread. This characteristic often leads to improved scalability.

\begin{figure}
    \centering
    \includegraphics[width=0.8\linewidth]{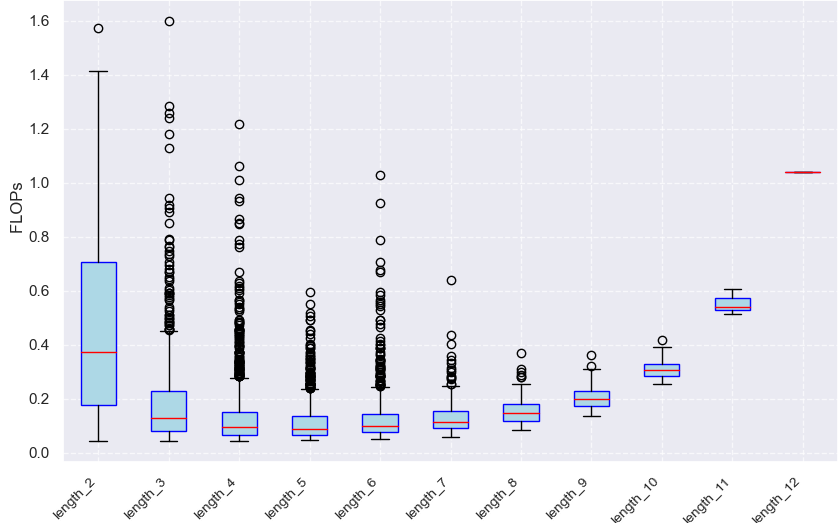}
    \caption{FLOPs for different solutions on a specific FC layer}
    \label{fig:flops_d}
\end{figure}

\subsection{Apply hardware-dependent compiler optimizations tailored to RISC-V architectures}\label{subsec:code-optimization}

In this section, we describe the compiler optimizations applied to the custom Einsum layers used by the T3F library, targeting RISC-V CPUs. The Einsum loop kernels are memory-bound, meaning their performance is primarily constrained by memory access throughput rather than computational throughput. To address this limitation, we implemented a set of compiler-level optimizations, including \textit{array packing, vectorization, register blocking, cache tiling, loop interchange, and parallelization}. These transformations were designed to mitigate memory bottlenecks and leverage the RISC-V architecture features effectively.
Our optimization process was guided by an analytical modeling approach and reasoning. The optimizations applied are listed below.

\subsubsection{Array packing optimization}\label{sub:cp}


To minimize the number of cache misses in the $G$ array reference of Listing \ref{lst:native}, array packing is applied. Array packing first, improves temporal locality and second, hardware prefetching is leveraged more effectively. This technique creates a new array at compile time with a new data layout; the new array's elements are stored in the exact order they are accessed in the Einsum layer. Specifically, the reordering follows the pattern:  

\[
G\_t\textbf{[m][rt][nt]}[rt\_1] \gets \text{reorder}(G\textbf{[rt][nt][m]}[rt\_1])
\]  


Listing \ref{lst:LA} shows the application of array packing to the Einsum code in Listing \ref{lst:native}. Note that array packing introduces a new loop kernel that initializes the new reordered array (not shown in Listing \ref{lst:LA}). This reordering can add overhead to our computation. Notably, while reordering may introduce computational overhead, this process is executed offline during compile time, ensuring that the optimization does not impose any runtime overhead. In an end-to-end application where we know the values of $G$ array reference at compile time we can change its layout. This is a well-known optimization that is used by many DNN compilers including IREE \cite{arcs}. After applying array packing, the two innermost loops in Listing \ref{lst:native} can be merged to one to simplify the code (Listing \ref{lst:LA}). 

It is important to note that the final implementation of array packing differs from the one presented in Listing \ref{lst:LA}. The layout of the $G\_t$ array must be adjusted based on the vectorization and register blocking optimization parameters, too. These adjustments are detailed in the subsequent subsections.

\begin{lstlisting}[language=C, frame=single, caption={Einsum Code after applying Array Packing}, label={lst:LA}]
for (int m = 0; m < mt; m++)
  for (int b = 0; b < bt; b++) 
    for (int r = 0; r < rt; r++)
      for (int k = 0; k < (nt*rt_1); k++){
       Output[m][b][r] += G_t[m][r][k]*Input[b][k];
      }
\end{lstlisting}

\subsubsection{Eliminate redundant reshape layers}

As illustrated in the \textit{T3F} example shown in Listing \ref{lst:python}, reshape layers are applied prior to each \texttt{Einsum} operation to ensure the $Output$ arrays have compatible shapes. By allocating these arrays as one-dimensional arrays, the reshape layers can be dynamically fused by efficiently managing array indexing.
Consider, for instance, the first reshape layer in Listing \ref{lst:python}. Here, $X$ array is a one-dimensional array of size $N$, where $N = b5 \times n5 \times r5$. Reshaping $X$ into the three-dimensional of shape $X[b5][n5][r5]$ can be accomplished without explicitly creating a new array. 



\subsubsection{Vectorization}\label{subsubsec:code-vev}

A key hardware feature of RISC-V processors is their vectorization engine (called V-extension). These processors are equipped with specialized hardware that support vector processing through intrinsics, enabling the parallel processing of up to $vl$ elements simultaneously. Efficient utilization of these intrinsics is crucial for achieving high performance.


The value of $vl$ is calculated as $vl = \text{vector\_length} / \text{data\_length}$, where \textit{vector\_length} represents the vector width of the machine (e.g., 256 bits for AVX2), and \textit{data\_length} corresponds to the size of each data element (e.g., 32 bits for single-precision floating-point numbers). On our RISC-V CPU, which has a 256-bit vector width, this calculation yields $vl = 8$.
Vectorization can be applied to one or two loops, leading to multiple vectorization options. To determine which loop(s) to vectorize, an analysis is provided for each loop in the loop kernel shown in Listing \ref{lst:LA} (assuming the data layout of the arrays is the one shown in Listing \ref{lst:LA}):  
\begin{enumerate}  
    \item \textit{Vectorizing the m-loop}: Vectorizing this loop involves loading and storing \textit{vl} elements in the \textit{m} dimension from the \textit{G\_t} and \textit{Output} arrays. This results in non-contiguous data access patterns, necessitating the use of the slow gather/scatter instructions. As a result, the data layout of \textit{Output} and \textit{G\_t} must change to improve data locality. Although this is not an issue for \textit{G\_t}, applying array packing for \textit{Output} would degrade performance as this array is not constant and this operation must be done at runtime 

    \item \textit{Vectorizing the b-loop}: Vectorizing this loop requires changing the layout of both \textit{Input} and \textit{Output} arrays. Similar to the previous case, array packing is needed for both arrays, degrading performance  

    \item \textit{Vectorizing the k-loop}: Listing \ref{lst:vec2} provides the vectorized code for this case. The main issue here is the use of horizontal addition operations (\texttt{vfredosum} in line 14), which negatively impact performance. As we can see, this approach also involves storing one element at a time (e.g., using \texttt{vfmv\_f\_s\_f32m1} in line 16) resulting in lower bandwidth utilization due to scalar stores \cite{RB2}  

    \item \textit{Vectorizing the r-loop}: Listing \ref{lst:vec1} provides the vectorized code for this case. The only issue with this approach is that the \textit{G\_t} array is not accessed contiguously. To resolve this, the layout of \textit{G\_t} is reordered from \texttt{\{m, rt, nt*rt\_1\}} to \texttt{\{m, rt/vl, nt*rt\_1, \textbf{vl}\}}, enabling contiguous data accesses. This layout change is achieved using array packing, which can be performed at compile time, as \textit{G\_t} is known before the execution of the code  
\end{enumerate}  

Based on the above analysis vectoring b and m-loops is inefficient as the overhead of changing the layout of \textit{Input} and \textit{Output} is significant. Between r and k-loops, it is more efficient to vectorize r-loop as we avoid the use of the extra and slow horizontal addition operations and we also store \textit{vl} multiple elements simultaneously. 

To sum up, the r-loop is chosen for vectorization. However, there is no r-loop in the final Einsum layer; as explained in Section 2, \( rt = 1 \), which necessitates vectorizing the k-loop instead. As mentioned, vectorizing the k-loop requires horizontal addition operations (e.g., \texttt{vfredosum} on line 14), resulting in a different microkernel compared to the earlier case. Moreover, no padding code is needed as both \texttt{rt} and \texttt{rt\_1} loops are always multiple of \textit{vl}. In Section \ref{subsec:vec_con}, we have selected the rank value to be a multiple of \textit{vl}.

\subsubsection{Register blocking and data reuse}


Register Blocking (RB), also known as 'Register-level Tiling' or 'unroll-and-jam,' is a key compiler optimization for developing efficient ukernels \cite{RB}. RB combines two compiler optimizations a) loop unroll and b) scalar replacement; by unrolling a number of loops, common array references are created in the loop body which are then replaced by variables; thus, the number of Load/Store (L/S) instructions is reduced.

\begin{lstlisting}[language=C, frame=single, caption={Einsum code after applying vectorization to k-loop}, label={lst:vec2}]
vfloat32m1_t out_0,core0,input0,z;
for(m=0; m<mt; m++){
 for(b=0; b<bt; b++){
  out_0 =vfmv_s_f_f32m1(0.0f,vl);
  for(k=0; k<(nt*rt_1); k++){
    C_in = m * nt*rt_1 + k
    in_in = b * nt*rt_1 + k
    c0 = vle32_v_f32m1(&G_t[C_in],vl);
    in0 = vle32_v_f32m1(&Input[in_in],vl);
    out0 = vfmacc_vv_f32m1(out0,c0,in0,vl);
  }
  out_i = m * bt + b;
  z = vfmv_s_f_f32m1(0.0f, vl);
  out0 = vfredosum_vs_f32m1_f32m1(out0,z,vl);
  // store 1 element back;
  Output[out_i] = vfmv_f_s_f32m1_f32(out0);
}
\end{lstlisting}

The application of RB on the Einsum kernel in Listing \ref{lst:vec1} is shown in Listing \ref{lst:vec_rb}. In Listing \ref{lst:vec_rb}, RB is applied to m and b-loops, with RB factors of 2 and 3 respectively. As it can be observed, by unrolling m and b-loops, common array references occur in the innermost loop body that are replaced by vector variables/registers. This reduces the number of L/S instructions, e.g., compared to Listing 5, the number of L/S instructions of \textit{G\_t} array is reduced by a factor of 3. The \textit{G\_t} array is accessed sequentially (Listing 6). To this end, we use array packing transformation at compile time and change the layout of \texttt{G\_t} from \texttt{\{mt, rt/vl, nt*rt, vl\}} to \texttt{\{mt, rt/(2*vl), nt*rt, \textbf{2*vl}\}}. This step follows the same logic as in the previous case.

\begin{lstlisting}[language=C, frame=single, caption={Einsum code after applying vectorization to r\_i-loop}, label={lst:vec1}, numbers=left, numberstyle=\tiny, stepnumber=1]
vfloat32m1_t out_0,core0,input0;
for(m=0; m<mt; m++){
 for(b=0; b<bt; b++){
  for(r=0; r<rt; r+=vl){
   C_index = m*rt*nt*rt_1 + r*nt*rt_1;
   out_0 = vfmv_s_f_f32m1(0.0f,vl);
   for(k=0; k<(nt*rt_1); k++){
    // G_t[]: changed layout of G_t from [mt,rt, nt*rt_1] to [mt, rt/vl, nt*rt_1, vl]
    c0 = vle32_v_f32m1(&G_t[C_index],vl);
    
    in_index = b*nt*rt_1 + k;
    in0 = vfmv_v_f_f32m1(Input[in_index],vl);
    out0 = vfmacc_vv_f32m1(out0,c0,in0,vl);
    
    C_index += vl; // sequencial access for G duo to array packing
    }
   out_i = m*bt*rt + b*rt + r;
   vse32_v_f32m1(&Output[out_i],out0,vl); // store vl elements back;
  }
 }
}
\end{lstlisting}

The primary challenges in RB involve selecting which loop or loops to unroll and determining the optimal unroll factor for each selected loop. For the general Einsum kernel where four nested loops exists, there are \( 2^4 \) possible loop unrolling combinations to evaluate as well as a big number of different unroll factor values. For instance, we may choose to unroll the m-loop alone or both the m- and k-loops simultaneously. Additionally, identifying the appropriate unroll factor for each selected loop is highly dependent on the specific CPU architecture, which also increase the exploration space.

To determine where to employ RB and their corresponding factors, we build upon and extend our prior work \cite{RB}. This process involves three straightforward and systematic steps.
First, we constrain the RB factor values based on the number of available hardware (HW) registers. This ensures that the HW registers allocated within the innermost loop body do not exceed the total number of available registers. The objective is to keep reused variables in the registers rather than spilling them into memory, thereby optimizing performance.
Second, for the remaining RB factor candidates, we estimate the number of L/S instructions required for each configuration. This step allows us to evaluate the impact of different RB factor choices on memory access patterns.
Finally, we select the configuration that minimizes the number of L/S instructions. This choice is guided by the primary goal of RB: to improve data reuse at the register level. Enhanced data reuse reduces the number of L/S instructions, leading to fewer memory accesses and improved efficiency. This rationale highlights the critical importance of minimizing L/S instructions as the objective function for RB.

When the RB factors do not perfectly divide the loops' upper bounds, padding ukernels are needed (line 42, 44 in Listing \ref{lst:vec_rb}). Padding ukernels handle the remaining iterations by processing the leftover elements after the main computation. Compared to our previous work \cite{RB}, we now account the L/S instructions occurred by the padding ukernels as well. This enhancement enables more accurate modeling of the overall computation.


The three RB above-mentioned steps are explained in detail below.
To determine the loops to apply RB and their respective factors, consider the vectorized code in Listing \ref{lst:vec1}. Initially, we assume that RB is applicable to all loops and our analysis will decide which loops to use and their blocking factors. Let $Rm$ denote the RB factor for the \texttt{m}-loop, and so forth. The three steps are further explained below:

\begin{lstlisting}[language=C, frame=single, caption={Einsum code after applying RB}, label={lst:vec_rb}]
for(m=0; m<floor(mt/2)*2; m+=2){ //Rm=2
 for(b=0; b<floor(bt/3)*3; b+=3)//Rb=3
  for(r=0;r<rt;r+=vl){
   C_index=m*rt*nt*rt_1 + r*nt*rt_1;
   out_0_0 = vfmv_s_f_f32m1(0.0f,vl); out_0_1 = out_0_0;
   out_1_0 = out_0_0;out_1_1 = out_0_0;
   out_2_0 = out_0_0;out_2_1 = out_0_0;
   for(k=0;k<(nt*rt_1);k++){
   //G_t_[]: changed layout of cores from [mt, rt/vl, nt*rt, vl] to [mt, rt/(2*vl), nt*rt, 2*vl]
    c0 = vle32_v_f32m1(&G_t[C_index],vl);
    c1 = vle32_v_f32m1(&G_t[C_index + vl],vl);
    //b
    in_index = b*nt*rt_1 + k;
    in0 = vfmv_v_f_f32m1(In[in_index],vl);
    out0_0 = vfmacc_vv_f32m1(out0,c0,in0,vl);
    out0_1 = vfmacc_vv_f32m1(out0,c1,in0,vl);
    //b+1
    in_index = (b+1)*nt*rt_1 + k;
    in0 = vfmv_v_f_f32m1(In[in_index],vl);
    out1_0 = vfmacc_vv_f32m1(out1_0,c0,in0,vl);
    out1_1 = vfmacc_vv_f32m1(out1_1,c1,in0,vl);
    //b+2
    in_index = (b+2)*nt*rt_1 + k;
    in0 = vfmv_v_f_f32m1(In[in_index],vl);
    out2_0 = vfmacc_vv_f32m1(out2_0,c0,in0,vl);
    out2_1 = vfmacc_vv_f32m1(out2_1,c1,in0,vl);

    C_index +=2*vl;
    }
   // m
   out_i = m*bt*rt + b*rt + r;
   vse32_v_f32m1(&Output[out_i],out0_0,vl);
   vse32_v_f32m1(&Output[out_i + rt],out1_0,vl);
   vse32_v_f32m1(&Output[out_i + 2*rt],out2_0,vl);
   // m + 1
   out_i = (m + 1)*bt*rt + b*rt + r;
   vse32_v_f32m1(&Output[out_i],out0_1,vl);
   vse32_v_f32m1(&Output[out_i + rt],out1_1,vl);
   vse32_v_f32m1(&Output[out_i + 2*rt],out2_1,vl);
  }
 }
 padding_ukernel_Rb(); // padding ukernel when bt in not multiple of 3
}
padding_ukernel_Rm(); // padding ukernel when mt in not multiple of 2
\end{lstlisting}

\noindent\textbf{Step1 - Constraint the RB factor values based on the number of HW registers.} As a first step, we form an equation that constrains the RB factors by relating them to the number of available HW vector registers (available.regs) (Equation \ref{eq:eq2}).

\begin{equation}\label{eq:eq2}
\text{Reg.Output} + \min(\text{Reg.In}, \text{Reg.G}) + 1 \leq \text{available.regs}
\end{equation}

where:  
i) \textit{Reg.Output} is the number of registers used for the \textit{Output} array, and $\text{Reg.Output} = Rm \times Rb \times Rr$,  
ii) \textit{Reg.In} is the number of registers used for the \textit{in} array, and $\text{Reg.In} = Rb \times Rk$, 
iii) \textit{Reg.G\_t} is the number of registers used for the \textit{G} array, and $\text{Reg.G} = Rm \times Rr \times Rk$. Depending on the generated RB factors, only one register is allocated for either the in array or the G array. This explains the presence of the min function in Equation~\eqref{eq:eq2} ; this is further explained in~\cite{RB}.
Thus, Equation~\eqref{eq:eq2} becomes:

\begin{equation}\label{eq:reg}
Rm \times Rb \times Rr + \min(Rb \times Rk, Rm \times Rr) + 1 \leq \text{available.regs}
\end{equation}

For example, in Listing \ref{lst:vec_rb}, unrolling $Rm$ by a factor of 2 and $Rb$ by a factor of 3 results in the allocation of six HW registers for $Output$ (lines 6–8) and two HW registers for $G$ (lines 11–12). In the code, rather than allocating $Rb=3$ registers for $In$, we allocate a single register and overwrite it $Rb$ times. This approach enables efficient utilization of the available hardware registers.

\noindent\textbf{Step2 - Generate equation for L/S instructions.}  
In this step, we calculate the number of L/S instructions required for the \textit{Output}, \textit{G\_t}, and \textit{Input} arrays. These calculations depend on the blocking factors and the structure of the computation. The total number of L/S instructions is expressed as the summation of the L/S instructions for each array reference:  

\begin{equation}\label{eq:ls}
   L/S\_instrs = L/S^{(Output)} + L/S^{(Input)} + L/S^{(G\_t)}
   \end{equation}


The third term in Equation~\eqref{eq:ls} is shown below:

   \begin{equation}\label{eq:eq4}
   L/S^{(G\_t)} = \frac{mt \cdot \lfloor bt/Rb \rfloor \cdot rt \cdot nt \cdot rt\_1}{vl} \\
   + L/S^{(padding\_ukernel())}
   \end{equation}

Regarding Equation~\eqref{eq:eq4}, the first term approximates the number of L/S instructions when the loops' upper bounds are perfectly divided by the RB factors, while the second term approximates the number of L/S instructions when the loops' upper bounds are not perfectly divided by the blocking factors, e.g., using the padding ukernel in line 42 at listing \ref{lst:vec_rb}. The first term is given by the product of the loops' upper bounds. 

In Equation~\eqref{eq:pad1}, we show the $L/S^{(padding\_ukernel())}$ when $bt$ is not perfectly divided by $Rb$.

    \begin{equation}\label{eq:pad1}
   L/S^{(padding\_ukernel\_bt())} = \frac{mt \cdot rt \cdot nt \cdot rt\_1}{vl} \cdot \delta(bt \mod Rb)
   \end{equation}
   
    Where the Kronecker delta function is defined as:
    
    \begin{equation}\label{eq:del}
    \delta(x) = 
        \begin{cases} 
        0 & \text{if } x = 0, \\
        1 & \text{if } x \neq 0.
        \end{cases}
    \end{equation}


The second term in Equation~\eqref{eq:ls} ($L/S^{(Input)}$) is given by Equation~\eqref{eq:eq5}.

\begin{equation}\label{eq:eq5}
   \begin{split}
   L/S^{(Input)} = &\frac{\lfloor mt/Rm \rfloor \cdot bt \cdot \lfloor rt/Rr \rfloor  \cdot nt \cdot rt\_1}{vl} +  L/S^{(padding\_ukernel)}
   \end{split}
   \end{equation}
   
The calculation for the \textit{Input} array is similar to the \textit{G\_t} array. The first term in Equation~\eqref{eq:eq5} approximates the number of L/S instructions when the RB factors perfectly divide the loops' upper bounds, while the second term approximates the number of L/S instructions when the RB factors do not perfectly divide the loops' upper bounds. The second term in Equation~\eqref{eq:eq5} is calculated similarly.

The first term of Equation~\eqref{eq:ls} is shown below:

   \begin{equation}\label{eq:eq3}
   \begin{split}
   L/S^{(Output)} = \frac{mt \cdot \lfloor bt/Rb \rfloor  \cdot \lfloor rt/Rr \rfloor}{vl} +  L/S^{(padding\_ukernel)}
   \end{split}
   \end{equation}
The calculation for the \textit{Output} array is similar to the \textit{G\_t} array. As we can see in line 16 of listing \ref{lst:vec1}, we store vl elements simultaneously in each iteration. Since we have $mt*bt*rt$ iterations, the above equation is derived. 
   




   

The equations for RB, when the \( k \)-loop is vectorized (Listing \ref{lst:vec2}), remain similar and are applied in the same manner. In this case, Equations~\eqref{eq:reg}, \eqref{eq:eq4} and \eqref{eq:eq5} remain unchanged. However, the number of stores for the \textit{Output} array need to be amended.

\noindent\textbf{Step3 - Select the solution that minimizes the number of L/S instructions.}  

In the final step, the RB parameters that minimize the number of L/S instructions are found using enumeration. For example, given 16 HW registers and $\{mt, bt,  rt, nt* rt\_1\}$ = $\{128, 32, 8, 8\}$, the calculated RB factors are $\{Rm, Rb,  Rr, Rk\} = \{4, 3,  1, 1\}$. 


\subsubsection{Loop tiling, loop parallelization, and loop interchange.}


In this Subsection, we demonstrate the application of three key optimizations, loop tiling, loop parallelization, and loop interchange. 
These optimizations are strongly interdependent, and thus the order in which they are applied significantly affects the resulting schedule/binary.
Let us examine the general code structure following the application of vectorization (Listing \ref{lst:vec1}). Regarding loop permutation, there are 4! possible permutations, as there are four loops in total. Although we have not studied all 4! different permutations, we selected the following two permutations, because they achieve efficient data access patterns and data parallelism: i) $\{mt, bt,  rt, nt*rt\_1\}$ and ii) $\{bt, mt,  rt, nt*rt\_1\}$. In Listing 
\ref{lst:vec1}, only the three outermost loops are data parallel; therefore, we set these loops as the outermost in both permutations.

Next, loop tiling and loop parallelization are applied to both loop permutations. Loop tiling is a critical optimization for memory-bound loop kernels (as in our case), because it reduces the number of data accesses to the slower main memory. Since loop tiling and loop parallelization are strongly interdependent, these two optimizations are applied together. Loop tiling is specifically tailored for the L2 cache (the last-level cache in our target platform), minimizing the number of data accesses in main memory.
The following steps outline the methodology used for loop tiling:

\textbf{Step 1:} We begin by selecting the loop permutation $\{mt, bt, rt, nt*rt\_1\}$ and checking if the data access patterns fit in L2 cache. If the data access patterns fit in L2 (or equivalently if Equation~\eqref{eq:tile1} is satisfied), then loop tiling is not necessary. In this case, `mt` loop is parallelized. 

    \begin{equation}\label{eq:tile1}
    \textit{T} \times \lceil (bt \cdot rt \cdot 4) / L2.way \rceil + 
    \textit{T} \times \lceil (rt \cdot nt \cdot rt\_1 \cdot 4) / L2.way \rceil + 
    \lceil (bt \cdot nt \cdot rt\_1 \cdot 4) / L2.way \rceil \leq L2.assoc
    \end{equation}
    
Here, \textit{T} represents the number of CPU threads (T equals to the number of physical CPU cores), while \textit{L2.assoc} and \textit{L2.way} denote the associativity and the number of ways of the L2 cache, respectively. The first/second/third term in Equation~\eqref{eq:tile1} relates to the number of L2 cache ways allocated for the Output / G\_t\ / Input array, respectively.

It is important to note that not all arrays need to fit entirely in L2 to utilize data reuse. For the \textit{Output} array, $bt*rt*4$ bytes per thread are sufficient (L2 is shared among all four cores). Similarly, for the \textit{Input} and \textit{G\_t} arrays, the required sizes are $bt*nt*rt\_1*4$ and $rt*nt*rt\_1*4$ bytes, respectively. 

Since each reused part should fit in L2, its size and the sizes of the other arrays’ data it accompanies should be constrained. To ensure the reused parts remain in the cache, they should first contain consecutive memory locations (which they do), and second, an empty cache line should be allocated for each different modulo (with respect to the cache size) of the arrays’ memory addresses. In simple terms, each of the three arrays should be assigned to separate cache ways.
In Equation~\eqref{eq:tile1}, not one, but T cache ways are used for \textit{Output}/\textit{G\_t} arrays because T parts of size bt*rt*4/rt*nt*rt\_1*4 will be loaded into the cache, and these parts are not stored into consecutive memory locations; to make sure that these parts remain in the cache, they are treated as separate tiles. This is why the first two terms in Equation~\eqref{eq:tile1} are multiplied by T. 
If Equation~\eqref{eq:tile1} is not satisfied, we proceed to the next step. 

\textbf{Step 2:} Select the loop permutation $\{bt, mt,  rt, nt*rt\_1\}$ and check if the data access patterns fit in L2. If the data access patterns fit in L2 (thus Equation~\eqref{eq:tile2} is satisfied), tiling is not applied. `bt` loop is parallelized in this case. If Equation~\eqref{eq:tile2} is not satisfied, we move to the next step. 

    \begin{equation}\label{eq:tile2}
    1 + \lceil (mt \cdot rt \cdot nt*rt\_1 \cdot 4) / L2.way \rceil + 
    T \times \lceil (nt*rt\_1 \cdot 4) / L2.way \rceil \leq L2.assoc
    \end{equation}
    
\textbf{Step 3:} Select the loop permutation $\{mt, bt,  rt, nt*rt\_1\}$, parallelize the `mt` loop and apply loop tiling to `bt` loop with tile size $Btl$ such that:

    \begin{equation}\label{eq:tile3}
    T \times \lceil (Btl \cdot rt \cdot 4) / L2.way \rceil + 
    T \times \lceil (rt \cdot nt*rt\_1 \cdot 4) / L2.way \rceil + 
    \lceil (Btl \cdot nt*rt\_1 \cdot 4) / L2.way \rceil \leq L2.assoc
    \end{equation}

If Equation~\eqref{eq:tile3} is not satisfied, the solution is deemed inefficient and discarded, and the next candidate from the remaining DS is evaluated. However, in this work, we did not encounter any such cases. The above method satisfies that the number of off-chip data accesses is minimized.

\section{Experimental Setup}\label{sec:experimental-setup}


To demonstrate the effectiveness of our approach, we evaluated it using various CNNs and LLMs models, i.e., i) CNNs: LeNet5, LeNet300, AlexNet, VGG, GoogleNet, Xception, and ResNet on different datasets, i.e., MNIST, CIFAR10, CIFAR100, and ImageNet since the type of dataset affects also the size of FC layers, and ii) LLMs: GPT2-Medium, GPT2-Large, GPT2-ExtraLarge, GPT3-Ada, GPT3-Curie, and GPT3-Davinci on the WebText dataset with embedding dimension equals to 50257.

The proposed methodology was evaluated on the SpacemiT K1, an 8-core 64-bit RISC-V chip running at 1.6GHz. Each core features a 32KB L1 cache, with a shared 1MB L2 cache for all eight cores. The experiments were conducted using the cluster 0 in Banana Pi BPI-F3 which includes four CPU cores, running the Bianbu-23.10 operating system. 


The proposed work explained in Section 4.3, is evaluated using the IREE compiler \cite{iree}, an MLIR-based framework widely regarded as one of the best in its class \cite{arcs}. The experiments were conducted using version $candidate-20240515.894$ of IREE using cross compile. The "embedding compilation" was used, an Ahead-of-Time (AOT) compilation employed to execute MLIR modules on the HAL device while minimizing runtime overhead. The compilation steps and parameters for this procedure are provided in the Appendix. The procedure takes as input an MLIR module that describes the custom Einsum kernel, based on the HLO dialect (also included in the Appendix).

Our approach is also compared against Pluto \cite{pluto}. All the applicable polyhedral optimizations provided by Pluto were enabled and the L2-cache size was specified using the appropriate flag. 
To ensure vectorization, GCC was configured with the -O3 optimization level along with the $-march=rv64gcv\_zvl256b$ flag, which identifies the presence of a 256-bit vector engine.



\section{Evaluation Results}\label{sec:results}


\begin{figure}
    \centering
    \includegraphics[width=1\linewidth]{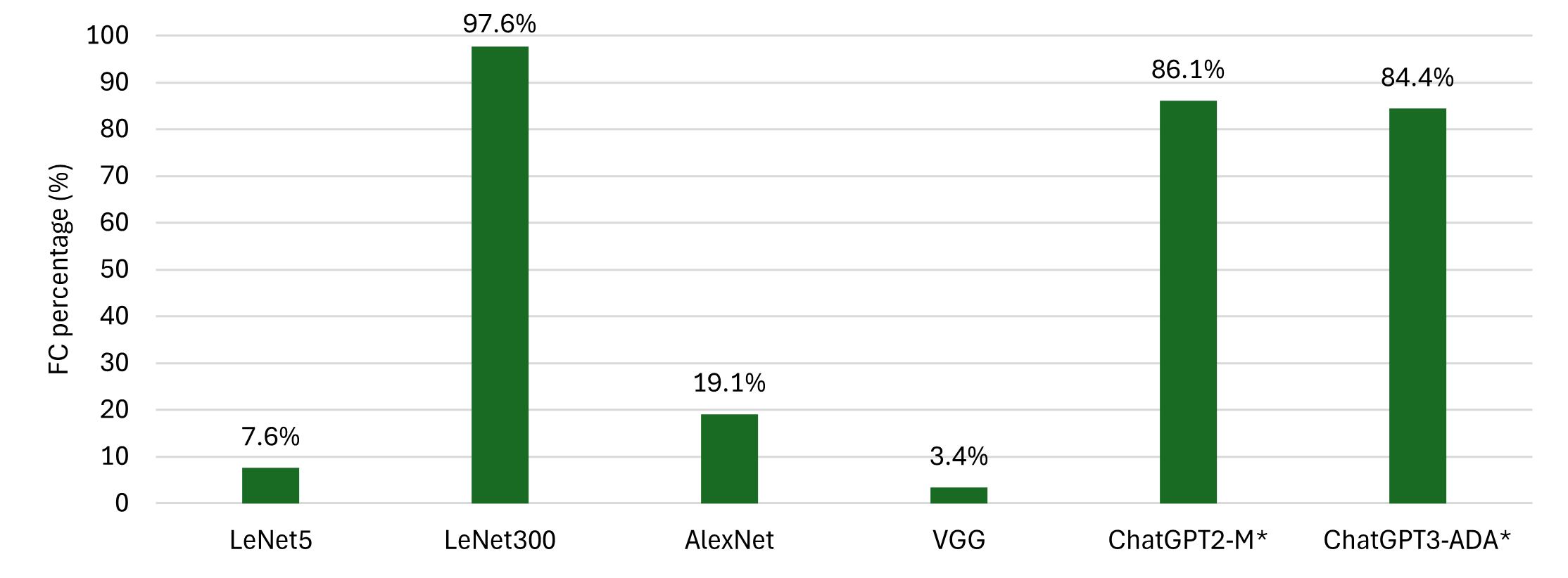}
    \caption{Percentage of total execution time spent in FC layers}
    \label{fig:fc-percentage}
\end{figure}

This section is divided into five parts. Subsection 6.1 shows the impact of the FC layers on inference time. Subsection 6.2 evaluates the methods used to reduce the DS (as described in Sections 4.1 and 4.2). Subsection 6.3 assesses the compiler optimization methods outlined in Section 4.3. Subsection 6.4 presents an end-to-end model evaluation compared to the case where LRF is not employed. Last, Subsection 6.5 provides a performance breakdown of the proposed compiler optimizations.

\subsection{FC layers Inference time analysis}
As showed in Figure~\ref{fig:params_flops}, FC layers can account for a significant portion of the total memory footprint in many architectures. To complement this analysis, we measured the execution time proportion of FC layers using the TensorFlow Lite benchmarking tool~\cite{tflite-benchmarking}, which provides layer-by-layer latency profiling on the target device. The results are presented in Figure~\ref{fig:fc-percentage}. We do not report results for all the studied models, as some employ only a single FC layer, making direct comparison unfair. Moreover, running very large LLMs on the target device was not feasible due to memory constraints (TensorFlow currently cannot convert them to TFLite due to the 2 GB flatbuffer limit).

The observed trends are closely related to the architectural composition of each model. For example, in LeNet300 model, the network consists almost entirely of FC layers and activation functions; consequently, the FC layers dominate the execution time, reaching 97.6\% of the total. Similarly, in LLM-family models, the core transformer blocks are primarily composed of FC layers (in the form of projection and feed-forward sub-layers), along with normalization and residual connections. In such cases, FC layers also represent a substantial fraction of the execution time (up to 86.1\% of the total).

In contrast, for the models that have many convolution layers, the impact of the FC part on inference time is lower. Convolutional layers are generally more computationally intensive (time-oriented), whereas FC layers tend to be more memory-oriented in these models, with high parameter counts but lower computational depth. As a result, the FC execution time proportion in CNNs is often smaller but can still be non-negligible.

Overall, Figure~\ref{fig:fc-percentage} highlights that FC layers can contribute significantly to execution time in a variety of architectures. While their impact varies depending on the model structure, they remain an important consideration for both memory and latency optimization.


\begin{table*}[]
\caption{Design space reduction for the studied CNN models\label{tab:CNNmodels}}
\centering
\resizebox{\textwidth}{!}{ 

\begin{tabular}{cccccccc}
\hline
\multirow{2}{*}{Model} & \multirow{2}{*}{Dataset} & \multirow{2}{*}{FC shapes} & \multicolumn{5}{c}{Number of solutions} \\ \cline{4-8} 
 &  &  & \begin{tabular}[c]{@{}c@{}}All initial\\ solutions\end{tabular} & \begin{tabular}[c]{@{}c@{}}Alignment\\ strategy\end{tabular} & \begin{tabular}[c]{@{}c@{}}Vectorization\\ constraint\end{tabular} & \begin{tabular}[c]{@{}c@{}}Initial layer\\ constraint\end{tabular} & \begin{tabular}[c]{@{}c@{}}Scalability\\ constraint\end{tabular} \\ \hline
\multirow{2}{*}{LeNet5} & \multirow{2}{*}{MNIST} & {[}400, 120{]} & 9.5E+08 & 1.2E+07 & 1.0E+03 & 2.2E+02 & 2.2E+02 \\
 &  & {[}120, 84{]} & 5.4E+06 & 1.1E+05 & 3.3E+02 & 5.6E+01 & 5.6E+01 \\ \hline
\multirow{2}{*}{LeNet300} & \multirow{2}{*}{MNIST} & {[}784 300{]} & 1.2E+10 & 6.8E+07 & 2.4E+03 & 5.7E+02 & 5.6E+02 \\
 &  & {[}300 100{]} & 1.1E+07 & 2.1E+05 & 4.5E+02 & 8.9E+01 & 8.9E+01 \\ \hline
\multirow{8}{*}{AlexNet} & \multirow{2}{*}{CIFAR10} & {[}4096, 2048{]} & 5.4E+20 & 5.4E+19 & 9.1E+03 & 4.1E+03 & 3.1E+03 \\
 &  & {[}2048, 2048{]} & 1.3E+19 & 1.9E+18 & 6.2E+03 & 2.6E+03 & 2.0E+03 \\ \cline{2-8} 
 & \multirow{3}{*}{CIFAR100} & {[}4096, 2048{]} & 5.4E+20 & 5.4E+19 & 9.1E+03 & 4.1E+03 & 3.1E+03 \\
 &  & {[}2048, 2048{]} & 1.3E+19 & 1.9E+18 & 6.2E+03 & 2.6E+03 & 2.0E+03 \\
 &  & {[}2048, 100{]} & 1.4E+08 & 2.5E+06 & 6.0E+02 & 1.1E+02 & 1.1E+02 \\ \cline{2-8} 
 & \multirow{3}{*}{ImageNet} & {[}9216, 4096{]} & 2.5E+25 & 3.5E+23 & 7.7E+04 & 3.9E+04 & 2.8E+04 \\
 &  & {[}4096, 4096{]} & 4.1E+22 & 6.6E+21 & 1.5E+04 & 6.5E+03 & 4.9E+03 \\
 &  & {[}4096, 1000{]} & 2.3E+14 & 6.5E+11 & 7.1E+03 & 2.2E+03 & 2.1E+03 \\ \hline
\multirow{8}{*}{VGG} & \multirow{2}{*}{CIFAR10} & {[}512, 512{]} & 1.1E+13 & 1.8E+12 & 1.1E+03 & 3.8E+02 & 3.2E+02 \\
 &  & {[}512, 256{]} & 4.2E+11 & 5.0E+10 & 6.3E+02 & 2.1E+02 & 1.3E+02 \\ \cline{2-8} 
 & \multirow{3}{*}{CIFAR100} & {[}512, 512{]} & 1.1E+13 & 1.8E+12 & 1.1E+03 & 3.8E+02 & 3.2E+02 \\
 &  & {[}512, 256{]} & 4.2E+11 & 5.0E+10 & 6.3E+02 & 2.1E+02 & 1.9E+02 \\
 &  & {[}256, 100{]} & 4.9E+06 & 1.7E+05 & 2.1E+02 & 4.1E+01 & 4.1E+01 \\ \cline{2-8} 
 & \multirow{3}{*}{ImageNet} & {[}25088, 4096{]} & 1.5E+26 & 2.8E+24 & 8.6E+04 & 3.6E+04 & 2.7E+04 \\
 &  & {[}4096, 4096{]} & 4.1E+22 & 6.6E+21 & 1.5E+04 & 6.5E+03 & 4.9E+03 \\
 &  & {[}4096, 1000{]} & 2.3E+14 & 6.5E+11 & 7.1E+03 & 2.2E+03 & 2.1E+03 \\ \hline
ResNet & ImageNet & {[}2048, 1000{]} & 3.6E+13 & 1.5E+11 & 4.6E+03 & 1.5E+03 & 1.5E+03 \\ \hline
GoogleNet & ImageNet & {[}1024, 1000{]} & 5.3E+12 & 3.5E+10 & 3.3E+03 & 1.0E+03 & 9.8E+02 \\ \hline
Xception & ImageNet & {[}2048, 1000{]} & 3.6E+13 & 1.5E+11 & 4.6E+03 & 1.5E+03 & 1.5E+03 \\ \hline
\end{tabular}
}
\end{table*}

\subsection{Assessing the effectiveness of DS reduction}

\begin{table*}[]
\caption{Design space reduction for the studied LLM models. The same dataset, WebText, is used for all models\label{tab:LLMmodels}}
\centering
\resizebox{\textwidth}{!}{ 

\begin{tabular}{ccccccc}
\hline
\multirow{2}{*}{Model} & \multirow{2}{*}{FC shapes} & \multicolumn{5}{c}{Number of solutions} \\ \cline{3-7} 
 &  & \begin{tabular}[c]{@{}c@{}}All initial\\ solutions\end{tabular} & \begin{tabular}[c]{@{}c@{}}Alignment\\ strategy\end{tabular} & \begin{tabular}[c]{@{}c@{}}Vectorization\\ constraint\end{tabular} & \begin{tabular}[c]{@{}c@{}}Initial layer\\ constraint\end{tabular} & \begin{tabular}[c]{@{}c@{}}Scalability\\ constraint\end{tabular} \\ \hline
\multirow{4}{*}{GPT2-Medium} & 24*4*({[}1024, 1024{]}) & 9.0E+15 & 1.6E+15 & 2.8E+03 & 1.0E+03 & 8.5E+02 \\
 & 24*({[}1024, 4096{]}) & 8.2E+18 & 5.6E+17 & 6.1E+03 & 2.4E+03 & 1.9E+03 \\
 & 24*({[}4096, 1024{]}) & 8.2E+18 & 5.6E+17 & 6.1E+03 & 2.4E+03 & 1.9E+03 \\
 & {[}1024, 50257{]} & 3.6E+04 & 1.7E+04 & 2.1E+03 & 1.0E+02 & 1.0E+02 \\ \hline
\multirow{4}{*}{GPT2-Large} & 36*4*({[}1280, 1280{]}) & 5.3E+16 & 2.5E+15 & 9.7E+03 & 3.6E+03 & 3.0E+03 \\
 & 36*({[}1280, 5120{]}) & 4.6E+19 & 9.5E+17 & 2.3E+04 & 9.2E+03 & 7.5E+03 \\
 & 36*({[}5120, 1280{]}) & 4.6E+19 & 9.5E+17 & 2.3E+04 & 9.2E+03 & 7.5E+03 \\
 & {[}1280, 50257{]} & 6.9E+04 & 3.2E+04 & 3.9E+03 & 1.7E+02 & 1.7E+02 \\ \hline
\multirow{4}{*}{GPT2-ExtraLarge} & 48*4*({[}1600, 1600{]}) & 3.1E+16 & 4.1E+14 & 1.8E+04 & 6.1E+03 & 5.4E+03 \\
 & 48*({[}1600, 6400{]}) & 2.1E+19 & 1.1E+17 & 4.7E+04 & 1.7E+04 & 1.4E+04 \\
 & 48*({[}6400, 1600{]}) & 2.1E+19 & 1.1E+17 & 4.7E+04 & 1.7E+04 & 1.4E+04 \\
 & {[}1600, 50257{]} & 8.5E+04 & 3.9E+04 & 4.9E+03 & 2.2E+02 & 2.2E+02 \\ \hline
\multirow{4}{*}{GPT3-Ada} & 12*4*({[}768, 768{]}) & 3.7E+15 & 5.9E+13 & 6.7E+03 & 2.7E+03 & 2.3E+03 \\
 & 12*({[}768, 3072{]}) & 2.4E+18 & 2.1E+16 & 1.6E+04 & 7.0E+03 & 5.6E+03 \\
 & 12*({[}3072, 768{]}) & 2.4E+18 & 2.1E+16 & 1.6E+04 & 7.0E+03 & 5.6E+03 \\
 & {[}768, 50257{]} & 5.4E+04 & 2.5E+04 & 3.1E+03 & 1.3E+02 & 1.3E+02 \\ \hline
\multirow{4}{*}{GPT3-Curie} & 24*4*({[}2048, 2048{]}) & 1.3E+19 & 1.9E+18 & 6.2E+03 & 2.6E+03 & 2.0E+03 \\
 & 24*({[}2048, 8192{]}) & 2.4E+22 & 2.1E+21 & 1.4E+04 & 5.9E+03 & 4.4E+03 \\
 & 24*({[}8192, 2048{]}) & 2.4E+22 & 2.1E+21 & 1.4E+04 & 5.9E+03 & 4.4E+03 \\
 & {[}2048, 50257{]} & 5.0E+04 & 2.3E+04 & 2.9E+03 & 1.8E+02 & 1.8E+02 \\ \hline
\multirow{4}{*}{GPT3-Davinci} & 96*4*({[}12288, 12288{]}) & 4.9E+29 & 5.3E+27 & 2.5E+05 & 1.3E+05 & 1.3E+05 \\
 & 96*({[}12288, 49152{]}) & 4.9E+33 & 2.9E+31 & 5.1E+05 & 2.6E+05 & 2.6E+05 \\
 & 96*({[}49152, 12288{]}) & 4.9E+33 & 2.9E+31 & 5.1E+05 & 2.6E+05 & 2.6E+05 \\
 & {[}12288, 50257{]} & 2.7E+05 & 1.3E+05 & 1.2E+01 & 1.2E+01 & 1.2E+01 \\ \hline
\end{tabular}
}
\end{table*}

The effectiveness of the DS reduction is shown in Table 1 for 27 FC layers of CNNs and in Table 2 for 24 FC layers of LLMs. The extremely small layers are not factorized as there is no point in applying LRF in these layers (these layers are not shown in Table 1 and Table 2). 

As it can be observed, the DS can become exceptionally large for certain layers, reaching up to $4.9 \times 10^{29}$ possible solutions. The size of the DS is determined by both the overall dimensions of the layer (larger arrays result in a larger DS) and the specific values of those dimensions (a higher number of factor combinations that match the dimension values leads to an increased number of potential solutions). For instance, the number 512 has more factor combinations that can be be multiplied to produce it compared to 1000.
This explains why a layer with dimensions $1024 \times 1000$ results in a smaller DS compared to a layer with dimensions $512 \times 512$ (Table 1).

The input-output shape alignment strategy, introduced in Section 4.1, reduces the DS based on the configuration length d (Equation~\eqref{eq:tf_Nov2}). A higher configuration length, which corresponds to a higher number of Einsum layers, results in a more significant reduction in the DS. However, the number of solutions with high d values is inherently smaller compared to those with low d values (this is because when the length increases the factors should be smaller). Overall, the alignment strategy achieves reductions in the DS ranging from $x2.1$ to $x92$ (GTP3-Davinci).

The vectorization constraint further reduces the DS by several orders of magnitude by retaining only those solutions with rank values that are multiples of eight. In contrast, the initial layer constraint has a minimal effect on pruning, as it excludes only those solutions with higher FLOPs than the original unfactorized layer. Finally, the scalability constraint, as shown in Tables 1 and 2, evaluates only the second scalability constraint proposed in Section 4.3. This constraint has a modest impact on the DS, since it prunes only the solutions with exceptionally small layers and configuration lengths exceeding four.



\subsection{Assessing the effectiveness of compiler optimizations}

In this subsection, we evaluate the compiler optimization approach outlined in Section 4.3, focusing on the custom Einsum layers used by the T3F library on RISC-V CPUs. Section 4.3 proposed several optimizations aimed at improving the performance of these layers. 
It is important to highlight that the code for the Einsum layers varies slightly depending on the configuration length of the solution, resulting in three distinct Einsum versions. 
The \textit{First Einsum} is characterized by \( \text{rt\_1} = 1 \), and it is utilized at the beginning of the T3F sequence. The \textit{Middle Einsum} features \( \text{rt} = \text{rt\_1} \neq 1 \), while the \textit{Final Einsum} has \( \text{rt} = 1 \) (Listing \ref{lst:native}).

A comprehensive experimental evaluation was conducted for all three Einsum versions using the SpacemiT K1 platform. The “CB0-CB7” IDs in Table \ref{tab:setup} represent eight distinct configuration shapes of Einsum kernels, each selected from the studied AI models. A rank value of eight was chosen, as optimization is more challenging for smaller rank values due to their low scalability.


\begin{table}[h!]
\caption{Selected sizes used for the three Einsum kernels}
\centering
\resizebox{\textwidth}{!}{%
\begin{tabular}{c|ccccc|ccccc|ccccc}
\toprule
\textbf{ID} & \multicolumn{5}{c|}{\textbf{First Einsum}} & \multicolumn{5}{c|}{\textbf{Middle Einsum}} & \multicolumn{5}{c}{\textbf{Final Einsum}} \\
\midrule
 & \textbf{mt} & \textbf{bt} & \textbf{nt} & \textbf{rt} & \textbf{FLOPs} 
 & \textbf{mt} & \textbf{bt} & \textbf{nt} & \textbf{rt, rt\_1} & \textbf{FLOPs} 
 & \textbf{mt} & \textbf{bt} & \textbf{nt} & \textbf{rt\_1} & \textbf{FLOPs} \\
\midrule
CB0 & 512 & 32  & 128 & 8   & 3.36E+07 & 48  & 224   & 2   & 8   & 2.75E+06 & 32  & 126  & 256 & 8   & 1.65E+07 \\
CB1 & 64  & 64  & 64  &     & 4.19E+06 & 64  & 3582  & 4   &     & 1.17E+08 & 64  & 64   & 128 &     & 8.39E+06 \\
CB2 & 128 & 1024 & 4   &     & 8.39E+06 & 96  & 128   & 14  &     & 2.20E+07 & 32  & 126  & 4   &     & 2.58E+05 \\
CB3 & 256 & 64   & 784 &     & 2.06E+08 & 64  & 64    & 32  &     & 1.68E+07 & 256 & 16   & 7   &     & 4.59E+05 \\
CB4 & 32  & 64   & 392 &     & 1.28E+07 & 256 & 128   & 4   &     & 1.68E+07 & 8   & 510  & 896 &     & 5.85E+07 \\
CB5 & 512 & 896  & 28  &     & 2.06E+08 & 32  & 9     & 7   &     & 2.58E+05 & 32  & 250  & 4   &     & 5.12E+05 \\
CB6 & 100 & 12   & 64  &     & 1.23E+06 & 4   & 16383 & 28  &     & 2.35E+08 & 124 & 9    & 16  &     & 2.86E+05 \\
CB7 & 16  & 4    & 150 &     & 1.54E+05 & 64  & 1020  & 28  &     & 2.34E+08 & 48  & 21   & 4   &     & 6.45E+04 \\
\bottomrule
\end{tabular}%
}
\label{tab:setup}
\end{table}

As noted, our approach was compared against two state-of-the-art tools: Pluto and IREE. 
Pluto is a source-to-source transformation tool that applies loop tiling, parallelization, loop interchange, and RB. Pluto does not directly support vectorization and relies on the compiler for this functionality.

IREE converts custom Einsum layers into Matrix-Matrix Multiplication (MMM) loop kernels, optimizing them through vectorization, parallelization, loop tiling, and RB. The transformation of the Einsum layer to MMM kernel is made during the "\textit{iree-stablehlo-to-stablehlo-preprocessing}" pass, as detailed in the Appendix. This pass will add three transpose layers for each array reference (Input, Output and $G$). It is important to note that the "\textit{iree-consteval-jit-globals}" pass eliminates the need for transposing the $G$ array. As mentioned in Section \ref{sub:cp}, $G$ is constant and can be stored in the desired memory order. Furthermore, IREE accompanies the MMM kernel with a packing and unpacking kernel, responsible to rearrange the storage of input and output arrays in memory to enable sequential accesses. This packing transformation adds an overhead and is necessary for accessing the kernel in consequentially memory locations in the MMM kernel. Finally, since some Einsum layers are relatively small and do not scale efficiently, both Pluto and IREE were evaluated using both 1 and 4 threads. The fastest execution time for each tool was selected for a fair comparison.

Figures~\ref{fig:first}, \ref{fig:midle}, and \ref{fig:final_einsum} present the GFLOP/s achieved for the first, middle, and final Einsum kernels, respectively. The x-axis in these figures corresponds to the eight different input sizes listed in Table \ref{tab:setup}. 



\textit{First Einsum Kernel:}
In Figure~\ref{fig:first}, the proposed method demonstrates superior performance compared to both Pluto and IREE. Our method averages 5.66 GFLOP/s, compared to IREE's 2.35 GFLOP/s and Pluto's 0.77 GFLOP/s. The proposed method is faster as we tackle more efficiently the aforementioned optimizations. The results indicate that Pluto is the slowest of the three approaches due to its failure to vectorize the kernel. Despite enabling relevant flags, Pluto depends on gcc to apply vectorization, which in this case was not effectively applied. Additionally, IREE's transformation of the Einsum layer into a MMM kernel using transpose layers, introduces more instructions providing less HW utilization compared to our method.

\begin{figure}
    \centering
    \includegraphics[width=1\linewidth]{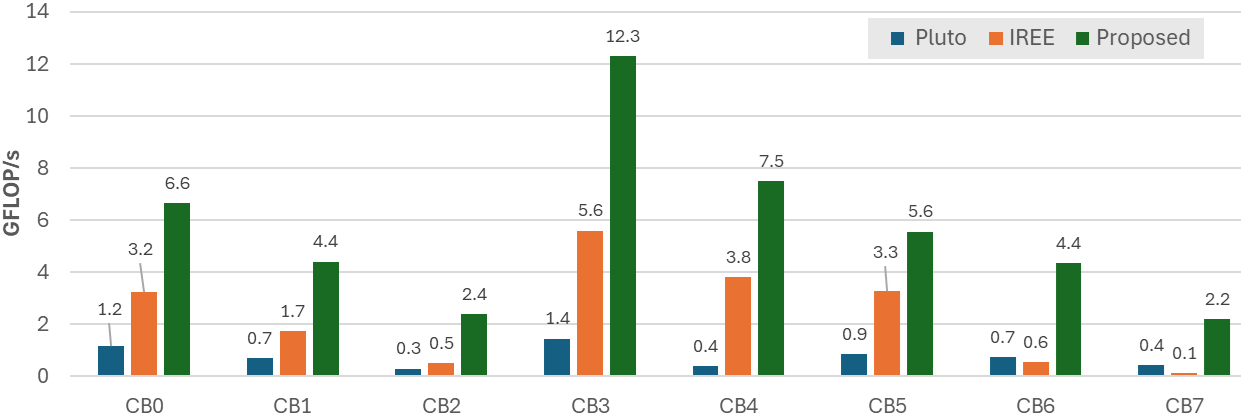}
    \caption{Evaluation of the first Einsum kernel}
    \label{fig:first}
\end{figure}

\begin{figure}
    \centering
    \includegraphics[width=1\linewidth]{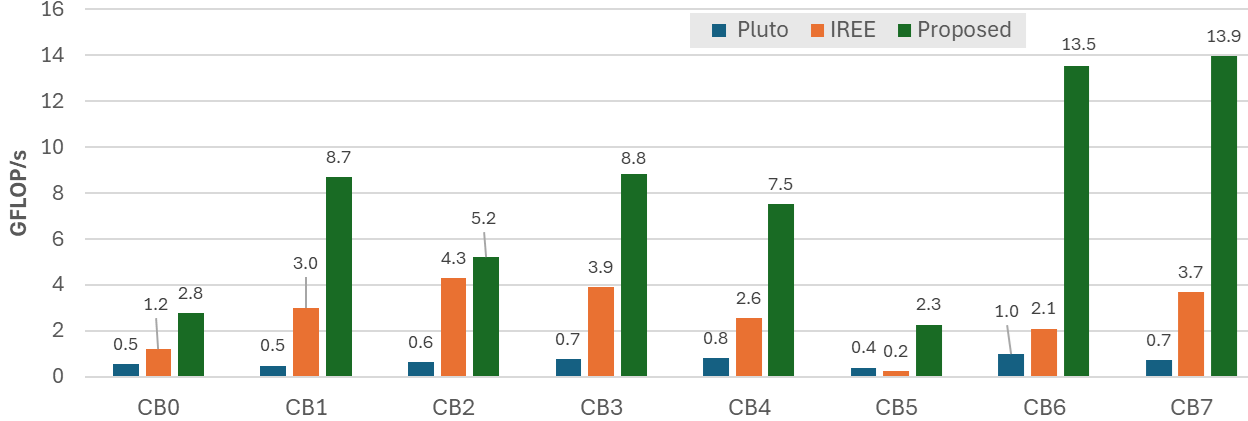}
    \caption{Evaluation of the Middle Einsum kernel}
    \label{fig:midle}
\end{figure}

\begin{figure}
    \centering
    \includegraphics[width=1\linewidth]{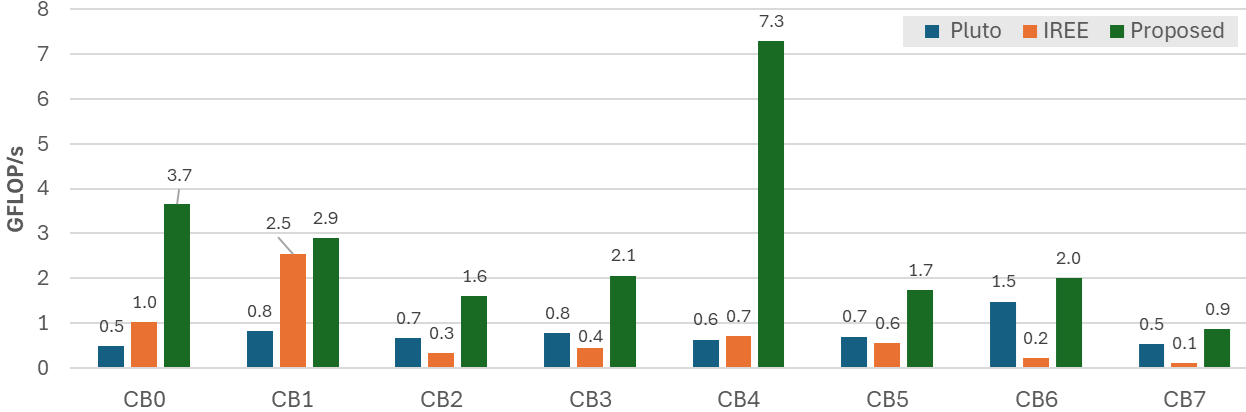}
    \caption{Evaluation of the Final Einsum kernel}
    \label{fig:final_einsum}
\end{figure}

\textit{Middle Einsum Kernel:} Figure~\ref{fig:midle} shows trends similar to those in Figure~\ref{fig:first}, as the first and middle Einsum kernels exhibit comparable characteristics. As we can see in Figure~\ref{fig:midle}, our method maintains its lead with an average of 7.84 GFLOPs, compared to IREE's 2.61 GFLOPs and Pluto's 0.64 GFLOPs.



Notably, cases such as CB5, CB6 and CB7 demonstrate our approach’s superior efficiency. In the first case, we can see a speedup of 12.5x compared to IREE. We attribute this to IREE's use of transposing and packing/unpacking layers. In this particular case, the solution involves only 2.58E+05 FLOPs, the smallest workload in the suite, rendering these optimizations largely ineffective. In CB6, a speedup of 5.7x is observed. We believe that this is due to the fact that we select the loop permutation $\{bt, mt,  rt, nt*rt\_1\}$ to fit data into L2-cache (and minimize DDR accesses). Similarly, in CB7, tiling the $bt$ loop further enhances performance, resulting in a 3.7× speedup.

\textit{Final Einsum Kernel:} Figure~\ref{fig:final_einsum} shows performance for the final Einsum kernel. In this kernel, we achieve an average of 2.76 GFLOPs, significantly outperforming IREE with 0.74 GFLOPs and Pluto with 0.76 GFLOPs. The proposed methodology demonstrates a notable reduction in HW utilization compared to previously kernel cases (First and Midle Einsum). This reduction can be attributed to two key factors: i) most of the layers are very small and they do not scale and ii) the vectorization process is not that efficient in the last Einsum kernel, as $k$ loop is vectorized (Listing 5).

Our method achieves between 1–14 GFLOPs, while the target hardware’s theoretical peak is 25.6 GFLOPs per CPU core. The achieved performance is lower than the theoretical peak primarily due to two reasons. First, the target loop kernel is memory-bound, and the memory hierarchy of this RISC-V processor is significantly slower than that of high-performance systems. Using the tool in \cite{smith}, we compared the bandwidth of the target embedded CPU against a high-performance Intel i9 CPU and observed 8x lower BW on the embedded platform.
Second, we intentionally used a small rank value (rank=8) in our experiments, as optimization is more challenging for small ranks due to their low scalability. 




\begin{figure}
    \centering
    \includegraphics[width=0.9\linewidth]{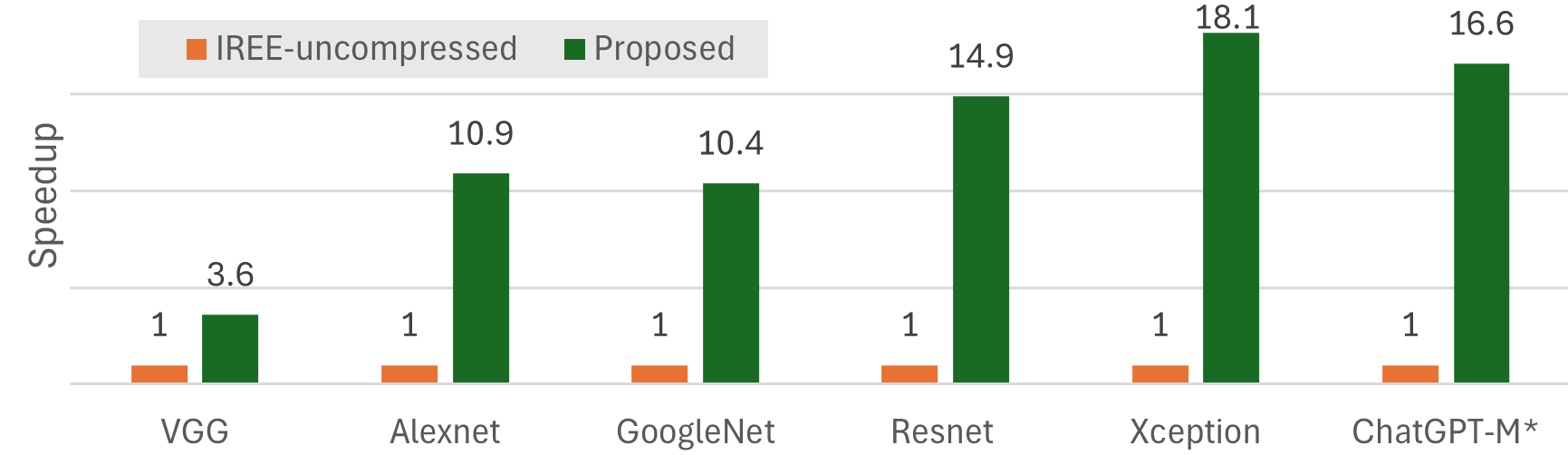}
    \caption{Speedup achieved over IREE for the initial uncompressed FC layers}
    \label{fig:app}
\end{figure}

\subsection{End-to-end evaluation}

In this subsection, we present an end-to-end evaluation of the studied models. We compare the proposed work against IREE but without applying LRF, across six AI models (Figure~\ref{fig:app}). 
This evaluation highlights the overall improvements achieved by incorporating LRF and utilizing our tool.
For the IREE benchmark, non-factorized FC layers were executed using the MMM kernel.

We report both the achieved speedup over IREE and the selected factorized solutions. As shown in the last column of Tables 1 and 2, multiple solutions are generated for each FC layer. For clarity, we selected for each layer the solution with the minimum FLOPs and a configuration length of two (comprising two Einsum layers). A fixed rank of eight was used for all solutions. The results are as follows: 

\begin{itemize}
    \item \textbf{ResNet}: The FC layer with shape of $[2048,  1000]$ was factorized into $[32 \times 64, 100 \times 10]$. To ensure minimal FLOPs, the shapes of the factorized layers follow the alignment strategy as proposed. For example, for a solution with rank-16 and core shapes $\{rt, nt, mt, rt\_1\}$, we get $\text{G}_1 = \{16, 32, 100, 1\}$ and $\text{G}_0 = \{1, 64, 10, 16\}$   

    \item \textbf{Xception}: Sharing the same FC layer dimensions as ResNet, $[2048,  1000]$ was factorized into $[32 \times 64, 25 \times 40]$

    \item \textbf{VGG}: The FC layer [512, 512] was factorized into [$16 \times 32$, $32 \times 16$]. The  FC layer [512, 256] was factorized into [$16 \times 32$, $16 \times 16$]. Finally, the FC layer [256, 100] was factorized into [$32 \times 8$, $10 \times 10$]

    \item \textbf{GoogleNet}: The FC layer [1024, 1000] was factorized into [$16 \times 64$, $40 \times 25$]

    \item \textbf{AlexNet}: The FC layer [4096, 2048] was factorized into [$64 \times 64$, $64 \times 32$]. The  FC layer [2048, 2048] was factorized into [$32 \times 64$, $64 \times 32$]. Additionally, the FC layer [2048, 10] was factorized into [$32 \times 64$, $5 \times 2$]

    \item \textbf{ChatGPT-M*.}: The  FC layer [1024, 1024] was factorized into [$16 \times 64$, $64 \times 16$]. The FC layer [4096, 1024] and [1024, 4096] was factorized into [$64 \times 64$, $64 \times 16$] and [$64 \times 16$, $64 \times 64$] respectively. In the final layer ([1024, 50257]), IREE could not compile it due to its large memory size, therefore, it is excluded from our results 
    
\end{itemize}

In Figure~\ref{fig:app}, our methodology (green bar) is compared to uncompressed IREE (orange bars). The y-axis represents
the speedup, while the x-axis lists the evaluated AI models. Execution times are normalized to
the leftmost blue bar for each model and the accumulated execution times for all layers are
considered. Note that accuracy tests were not conducted for these layers; the reported speedup
reflects improvements solely in the FC portions of the models.
On average, our approach achieves a 12x speedup compared to IREE. However, for the VGG
model, the speedup is less pronounced. This is because VGG contains smaller layers, which do not
benefit significantly from TTD. TTD is most effective in optimizing performance for larger layers.

\subsection{Performance Breakdown}\label{s:64}

In Figure~\ref{fig:gcc}, the performance of factorized FC layers is compared across different optimization stages: unoptimized using GCC O3 optimization level(blue), optimized after apply vectorization and packing (orange), optimized after apply RB and Tiling (green) and our proposed fully optimized approach after parallelization (purple). The y-axis represents the speedup, while the x-axis lists the compared AI models. We used the same combinations shapes with subsection 6.4, but in this case rank sixteen was used. Execution times are normalized to the leftmost blue bar for each model.

On average, our approach achieves a 37x speedup compared to GCC. Best performance improvements are noticed after applying packing and vectorization (9x performance improvements compared to gcc). Tiling and RB  provided 2x speedup on average. In the selected configuration shapes cache tiling did not improve performance as arrays fit in L2 cache. Finally, parallelization provided 1.7x speedup on average.

\begin{figure}
    \centering
    \includegraphics[width=0.9\linewidth]{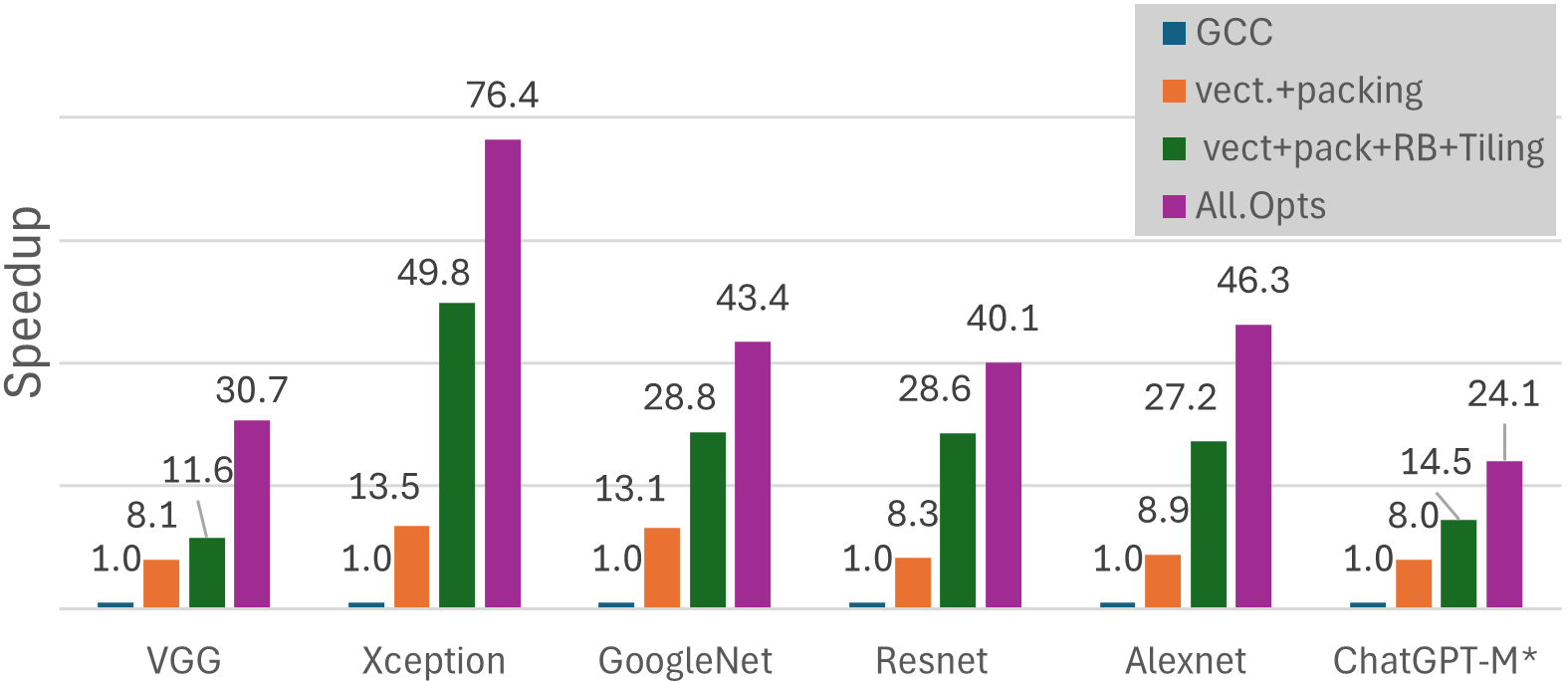}
    \caption{Performance Gains Across Factorized FC Layers Through Progressive Optimizations}
    \label{fig:gcc}
\end{figure}

\section{Conclusions}\label{sec:conclusion}

This paper presents an end-to-end LRF DSE methodology, along with a specialized design tool, for optimizing FC layers on RISC-V processors. The first contribution of this work is a novel shape alignment strategy that significantly prunes the DS by excluding decomposition shapes incapable of achieving low FLOPs. This strategy is general and applicable to any TTD DSE method. Our experiments demonstrate that this approach reduces the design space by a factor from 2.1x up to 92x, making the optimization process significantly more efficient.

However, low FLOPs or memory usage do not always guarantee efficient inference times. To address this, the second contribution introduces heuristics tailored to RISC-V platforms that prune LRF solutions incapable of achieving low inference latency. Solutions failing to meet vectorization and scalability constraints are systematically eliminated, ensuring only high-performance candidates are retained. Finally, the most critical compiler optimizations are applied to further enhance the performance of custom T3F layers. These optimizations effectively reduce inference latency and maximize computational efficiency on RISC-V processors. Moreover, the proposed compiler techniques are extensible to other processor architectures, highlighting their broader applicability.
Experimental results demonstrate that our TT-decomposed layers run on average 3× faster than IREE and 8× faster than Pluto when executed the same factorized layers. Furthermore, compared to the original uncompressed model running on IREE, our factorized version achieves an average speedup of 12x.


\section*{Acknowledgment}
This work is part of R-PODID project, supported by the Chips Joint Undertaking and its members, including the top-up funding by National Authorities of Italy, Turkey, Portugal, The Netherlands, Czech Republic, Latvia, Greece, and Romania under grant agreement No 101112338.

\bibliographystyle{ACM-Reference-Format}
\bibliography{t3f}

\newpage

\section*{Appendix}

In this section, we provide details about the compilation flow in IREE. Listing \ref{lst:IREE} presents two key components: i) the MLIR function fed to IREE for the CB5 Middle Einsum operation from Table 3, and ii) the corresponding IREE compilation command. As shown in line 3 of the listing, the core values are constant and initialized to 0.1.

To verify the correctness of the compilation process, we ran a simple example and compared the results from IREE with our own approach. Listing \ref{lst:IREE2} shows the MLIR code after the iree-stablehlo-to-stablehlo-preprocessing pass. Following this pass, the Cores (lines 5-8), Inputs (lines 10-14), and Outputs (lines 22-25) are transposed and reshaped. The Einsum operation is then transformed into a Matrix-Matrix Multiplication (MMM) kernel, as seen on line 17.

\begin{tcolorbox}[colframe=black, colback=white, width=\textwidth]
    \begin{minipage}[t]{\textwidth}
        \begin{lstlisting}[language=python, frame=single, numbers=left, stepnumber=1, numbersep=5pt, xleftmargin=15pt, label={lst:IREE-MLIR}]
#1)Input MLIR (input.mlir) for CB5 Middle Einsum of table 3
func @einsum(%x:tensor<9x7x8xf32>)->tensor<32x9x8xf32> {
    %G = stablehlo.constant dense<1.000000e-01>:
    tensor<8x7x32x8xf32>

    %y=stablehlo.einsum %G,%x,config="rnmk,bnk->mbr":
    (tensor<8x7x32x8xf32>, tensor<9x7x8xf32>)
    -> tensor<32x9x8xf32>
    return %y : tensor<32x9x8xf32>
}
        \end{lstlisting}
    \end{minipage}

    \vspace{10pt}

    \begin{minipage}[t]{\textwidth}
        \begin{lstlisting}[language=bash, frame=single, numbers=left, stepnumber=1, numbersep=5pt, xleftmargin=15pt, caption={IREE Compilation}, label={lst:IREE}]
#2) IREE Compilation Command
iree-compile \
  --iree-input-type=stablehlo \
  --iree-hal-target-backends=llvm-cpu \
  --iree-llvmcpu-link-embedded=false \
  --iree-llvmcpu-target-triple=riscv64-unknown-linux-gnu\
  --iree-llvmcpu-target-cpu=generic-rv64 \
  --iree-llvmcpu-target-abi=lp64d \
  --iree-llvmcpu-target-cpu-features=+f,..,+zvl256b,+v \
  --riscv-v-fixed-length-vector-lmul-max=8 \
  --iree-llvmcpu-link-static \
  --iree-llvmcpu-static-library-output-path=out_static.o \
  input.mlir -o out.vmfb
        \end{lstlisting}
    \end{minipage}

    \vspace{10pt}
\end{tcolorbox}

\begin{tcolorbox}[colframe=black, colback=white, width=\textwidth]
    \begin{minipage}[t]{\textwidth}
        \begin{lstlisting}[language=python, frame=single, numbers=left, stepnumber=1, numbersep=5pt, xleftmargin=15pt, caption={MLIR after "iree-stablehlo-to-stablehlo-preprocessing" pass}, label={lst:IREE2}]
func @einsum(%x:tensor<9x7x8xf32>)->tensor<32x9x8xf32> {
    %G = stablehlo.constant dense<1.000000e-01>
    #  Traspose reshape of Cores
    : tensor<8x7x32x8xf32>
    %1 = stablehlo.transpose %G, dims = [0, 2, 1, 3]:
    (tensor<8x7x32x8xf32>) -> tensor<8x32x7x8xf32>
    %2 = stablehlo.reshape %1 : (tensor<8x32x7x8xf32>)
    ->tensor<256x56xf32>

    # Traspose reshape Input
    %3 = stablehlo.transpose %x, dims = [1, 2, 0]:
    (tensor<9x7x8xf32>) -> tensor<7x8x9xf32>
    %4 = stablehlo.reshape %3 : (tensor<7x8x9xf32>)->
    tensor<56x9xf32>

    # Matrix-Matrix Multiplication kernel
    %5 = stablehlo.dot %2, %4 :
    (tensor<256x56xf32>, tensor<56x9xf32>)
    -> tensor<256x9xf32>

    # Traspose reshape Output
    %6 = stablehlo.reshape %5 : (tensor<256x9xf32>)
    -> tensor<8x32x9xf32>
    %y = stablehlo.transpose %6, dims = [1, 2, 0] :
    (tensor<8x32x9xf32>) -> tensor<32x9x8xf32>
    return %y : tensor<32x9x8xf32>
}
        \end{lstlisting}
    \end{minipage}
\end{tcolorbox}


\end{document}